\documentclass[iicol,sn-basic]{sn-jnl}

\usepackage{graphicx}%
\usepackage{multirow}%
\usepackage{amsmath,amssymb,amsfonts}%
\usepackage{amsthm}%
\usepackage{mathrsfs}%
\usepackage[title]{appendix}%
\usepackage{xcolor}%
\usepackage{textcomp}%
\usepackage{manyfoot}%
\usepackage{booktabs}%
\usepackage{algorithm}%
\usepackage{algorithmicx}%
\usepackage{algpseudocode}%
\usepackage{listings}%
\usepackage{subcaption} 
\usepackage{adjustbox}

\newtheorem{theorem}{Theorem}
\newtheorem{definition}{Definition}%

\newtheorem{lemma}{Lemma}
\newtheorem*{theorem-restate}{Theorem}
\newtheorem*{lemma-restate}{Lemma}

\newcommand\norm[1]{\left\lVert#1\right\rVert}

\begin{document}

\title[Article Title]{Dynamic Policy Fusion for User Alignment Without Re-Interaction}


\author*[1]{\fnm{Ajsal Shereef} \sur{Palattuparambil}}\email{a.palattuparambil@deakin.edu.au}

\author[2]{\fnm{Thommen George} \sur{Karimpanal}}\email{thommen.karimpanalgeorge@deakin.edu.au}
\equalcont{These authors contributed equally to this work.}

\author[1]{\fnm{Santu} \sur{Rana}}\email{santu.rana@deakin.edu.au}
\equalcont{These authors contributed equally to this work.}

\affil[1]{\orgdiv{A2I2}, \orgname{Deakin University}, \orgaddress{\street{Waurn Ponds}, \city{Geelong}, \postcode{3216}, \state{Victoria}, \country{Australia}}}

\affil[2]{\orgdiv{School of IT}, \orgname{Deakin University}, \orgaddress{\street{Waurn Ponds}, \city{Geelong}, \postcode{3216}, \state{Victoria}, \country{Australia}}}


\abstract{
  Deep Reinforcement Learning policies, although optimal in terms of task rewards, may not align with the personal preferences of human users. To ensure this alignment, a na\"ive solution would be to retrain the agent using a reward function that encodes the user's specific preferences. However, such a reward function is typically not readily available, and as such, retraining the agent from scratch can be prohibitively expensive. We propose a more practical approach - modify the already trained policy to accommodate the user-specific needs with the help of human feedback. To this end, we infer the user's intent through trajectory-level feedback and combine it with the trained task policy via a theoretically grounded dynamic policy fusion approach. Our approach collects human feedback on the same trajectories used to learn the task policy. As a result, it does not require any additional interactions with the environment. This makes it a zero-shot approach regarding environment interactions. We empirically demonstrate in a number of environments that our proposed dynamic policy fusion approach consistently achieves the intended task while simultaneously adhering to user-specific needs.}

\keywords{Reinforcement Learning, User Alignment, Dynamic Policy Fusion, Human in the Loop.}

\maketitle
\section{Introduction}
\label{sec:intro}

Reinforcement learning (RL) has achieved remarkable success across a variety of domains, including games \citep{mnih2013playing}, robotics \citep{han2023survey}, healthcare \citep{yu2021reinforcement}, and recommendation systems \citep{sun2018conversational}. Standard RL agents are optimised to maximise a predefined reward signal, which typically encodes task efficiency. While this approach ensures task completion, the resulting agent behaviour may not align with the subjective preferences of human users. For instance, in a driving context, Reinforcement Learning (RL) agent trained to minimise travel time might adopt aggressive manoeuvres, whereas a human passenger would likely prefer a smoother, safer driving style. Since such personal preferences are rarely captured in the original reward function, aligning agent behaviour with user expectations remains a critical challenge.

A common approach involves retraining the agent with a reward function modified to incorporate user preferences. However, this is often impractical; designing such personalised reward functions is non-trivial, and retraining from scratch demands substantial additional interaction with the environment, which can be costly or infeasible. A more pragmatic alternative is to adapt a pre-trained policy which we refer to as the \emph{task-specific policy} so that it continues to achieve the primary task goal while also accommodating the user's intent.

Our framework personalises a pre-trained task-specific policy by inferring user intent from trajectory-level feedback. Critically, this feedback is applied to the same trajectories generated during the training of the task policy itself, enabling a zero-shot approach that requires no new environment interaction. From this feedback, we construct an auxiliary \emph{intent-specific policy} that captures user preferences on \emph{how} the task should be performed. The final, personalised policy is formed by fusing the original \emph{task-specific policy} with this new \emph{intent-specific policy}. A key challenge in this process is to balance the two objectives: blindly following user intent could compromise task success, while ignoring it defeats the purpose of personalisation. To resolve this, we introduce \emph{dynamic policy fusion} mechanism that adaptively modulates the influence of each policy based on the current state, ensuring both task completion and alignment with user preferences.

The main contributions of this paper are as follows:
\begin{itemize}
    \item We propose a \textbf{zero-shot personalisation framework} that adapts pre-trained RL policies to user-specific preferences using trajectory-level feedback, eliminating the need for additional environment interaction.
    \item We introduce a \textbf{policy fusion formulation} with formal guarantees, including bounded divergence from the original task policy and an invariability property for identical constituent policies.
    \item We develop a mechanism to dynamically adapt the influence of the intent policy, ensuring the final behaviour robustly balances task completion with user alignment.
\end{itemize}

This paper is an extension of our preliminary work presented in \cite{ajsal2024personalisation}. We expand upon the initial research in three key aspects:
\begin{itemize}
    \item We provide a comprehensive theoretical analysis of the policy fusion mechanism, establishing formal guarantees for its behaviour.
    \item Establishing the necessity of dynamic policy fusion through detailed empirical studies.
    \item We broaden the experimental evaluation to include more diverse human preference models, validating the robustness and versatility of our framework.
\end{itemize}

The rest of this paper is organised as follows: Section~\ref{sec:related} reviews related work and Section~\ref{sec:background} details the necessary preliminaries. Section~\ref{sec:methodology} presents our proposed approach, followed by our experimental evaluations in Section~\ref{sec:experiments}. Finally, Section~\ref{sec:limitation} discusses the limitations of our work and Section~\ref{sec:conclusion} concludes the paper.

\section{Related Work}
\label{sec:related}
Learning from human feedback is of particular interest in the RL community as it leverages human knowledge during the learning process, offering several benefits. Firstly, it improves the efficiency of the system in terms of sample requirements as well as overall performance \citep{guan2020explanation}. Secondly, leveraging human feedback has been shown to enable RL agents to solve complex tasks that are otherwise challenging to manually specify through conventional reward functions \citep{lee2021pebble,christiano2017deep}. Consequently, there exist a number of approaches aimed at leveraging human feedback for learning. 

A number of works in interactive RL \citep{torrey2013teaching,macglashan2017interactive} examine how agents can learn from state-action level human feedback. However, such methods are limited, as providing state-action-level feedback is non-intuitive and cognitively taxing. Preference-based RL \citep{christiano2017deep,lee2021pebble}, uses trajectory-level human preference feedback, using which a corresponding reward function is learned. However, following the learning of the reward function, this approach may require the agent to interact with the environment again to learn the corresponding policy. As our present work focuses on developing a zero-shot approach, to estimate human intent, we instead use Return Decomposition for Delayed Rewards (RUDDER) \citep{arjona2019rudder}, an Long Short Term Memory (LSTM)-based approach for credit assignment by directly approximating the Q-values from agent trajectories. These Q-values are then translated into the required intent-specific policy. There are more recent work addressing credit assignment such as \cite{patil2020align,zhang2024interpretable}, but we resort to RUDDER for our work due to it's simplicity.


In the context of policy fusion, a closely related work by \cite{sestini2021policy} combines policies with different fusion methods. However, the described policy fusion methods are static and may result in the dominance of one of the constituent policies. Policy fusion has been explored previously in \cite{haarnoja2018composable} and \cite{hunt2019composing}. Such works study the concept of learning different policies independently using soft Q-functions, each with its own reward function. These individual policies are later combined, leading to the emergence of new behaviours in robotic manipulations. Our work views policy fusion through a lens of personalisation, requiring the fusion to satisfy certain constraints. In addition, unlike existing policy fusion methods, our dynamic policy fusion approach prevents the over-dominance of the component policies.

The overarching aim of our work is to personalise existing task policies using trajectory level human feedback without the need for additional environment interactions. The goal of personalisation could possibly be viewed from a multi-objective reinforcement learning (MORL) \citep{hayes2022practical}  or constrained RL \citep{tessler2018reward} perspective. 
MORL embodies personalisation by determining which objectives take precedence, and addressing conflicting objectives \citep{basaklar2022pd} through Pareto-optimal solutions.  
In ~\cite{mo2018personalizing}, a personalised dialogue system is developed using transfer learning, adapting common dialogue knowledge from a source domain to a target user, whereas ~\cite{bodas2018reinforcement} aims to personalise non-player characters (NPCs) to human skill levels to enhance player engagement. This approach designs a composite scalar reward function that implicitly matches the NPC's skill level to that of the human player. However, it necessitates additional environment interaction and NPC training with the engineered reward function. Unlike these works, our approach does not require any additional interaction or specifically crafting the reward function. We believe these benefits make the present work particularly relevant for real-world scenarios.


\section{Preliminaries and Background}
\label{sec:background}
We briefly survey some of the related backgrounds that form the basis for our work. We refer to the policy trained to solve the task as the \emph{task-specific policy} $(\pi_\phi)$ and the policy that captures human intent as the \emph{intent-specific policy}$(\pi_\psi)$.

\subsection{Reinforcement Learning}
The policies considered in this work are learned by solving tasks represented as individual Markov Decision Processes (MDPs) \citep{puterman2014markov}  $\mathcal M = (\mathcal S, \mathcal A,  P,  R, \gamma)$ where $\mathcal S$, and $\mathcal A$ are the state and action spaces. $P$ is the transition function that captures the transition dynamics of the environment. $R$ is the reward function and $\gamma$ is the discount factor. In each task, $R$ can vary, while the remaining components stay the same. At timestep $t$, the agent in state $s_t \in \mathcal{S}$ takes an action $a_t \in \mathcal{A}$ and obtains a reward $R(s_{t+1}, s_t, a_t)$ and moves to state $s_{t+1}$ according to the transition function $P(s_{t+1}|s_t,a_t)$. A policy $\pi(a_t|s_t)$ outputs the probability of taking an action $a_t$ from a given state $s_t$.  The episodic discounted return is $G_t = \sum_{t=0}^{T}\gamma^tr_t$, where $\gamma$ specifies how much the future reward is discounted and $T$ is the total number of timesteps. The agent's objective is to maximise the future expected reward $\mathbb{E}[G_t]$ by learning a 
Q-function $Q(s_t,a_t)$, which estimates the expected cumulative reward.

\subsection{RUDDER}
\label{sec:human_leanring}



 RUDDER \citep{arjona2019rudder} addresses the problem of credit assignment and learning from sparse rewards. In RUDDER, an LSTM network analyses the entire trajectory data with a score received at the end to estimate Q-values for individual state-action pairs within the trajectory. These Q-values represent the expected future reward for taking an action in a given state. RUDDER also showed a densified reward can be formulated from the human feedback score as follows,
\begin{equation}
\label{eq:reward_redistribution}
E[r_{t+1}|s_t,a_t,s_{t-1},a_{t-1}] = Q'(s_t,a_t) - Q'(s_{t-1},a_{t-1}),
\end{equation}
where, $Q'(s_t,a_t)$ represents the LSTM-generated Q-values corresponding to the state-action pair at the specific time step $t$.

We use the RUDDER framework to learn the Q-values corresponding to human scores and later convert the Q-values to policies as described in Sections \ref{sec:intent_learning} and \ref{sec:policy_constrcution}. The densified reward as in \eqref{eq:reward_redistribution} is used to modulate the influence of the policy as discussed in Section \ref{sec:dynamic_policy_fusion}. For more details of the training of LSTM, please refer to Section B of the supplementary.

\section{Methodology}
\label{sec:methodology}
We construct our proposed approach for personalisation by first inferring the intent-specific policy via trajectory-level human feedback using RUDDER, followed by a dynamic policy fusion mechanism that automatically maintains a balance between the inferred policy and a trained task-specific policy.

We assume that the trajectory data used to train task-specific policy is accessible, and reusing this data allows us to perform zero-shot personalisation i.e., without collecting new environment interactions. A subset of this data is sampled with personalised human feedback scores with more desirable trajectories assigned higher scores. We then infer the human intent (intent-specific policy) using RUDDER (Section \ref{sec:human_leanring}) and dynamically fuse it with the task-specific policy. The overview of our method is illustrated in Fig \ref{fig:overview}. The subsequent sections provide in-depth details of the components of our personalisation approach.

\begin{figure*}
\includegraphics[width=0.9\textwidth]{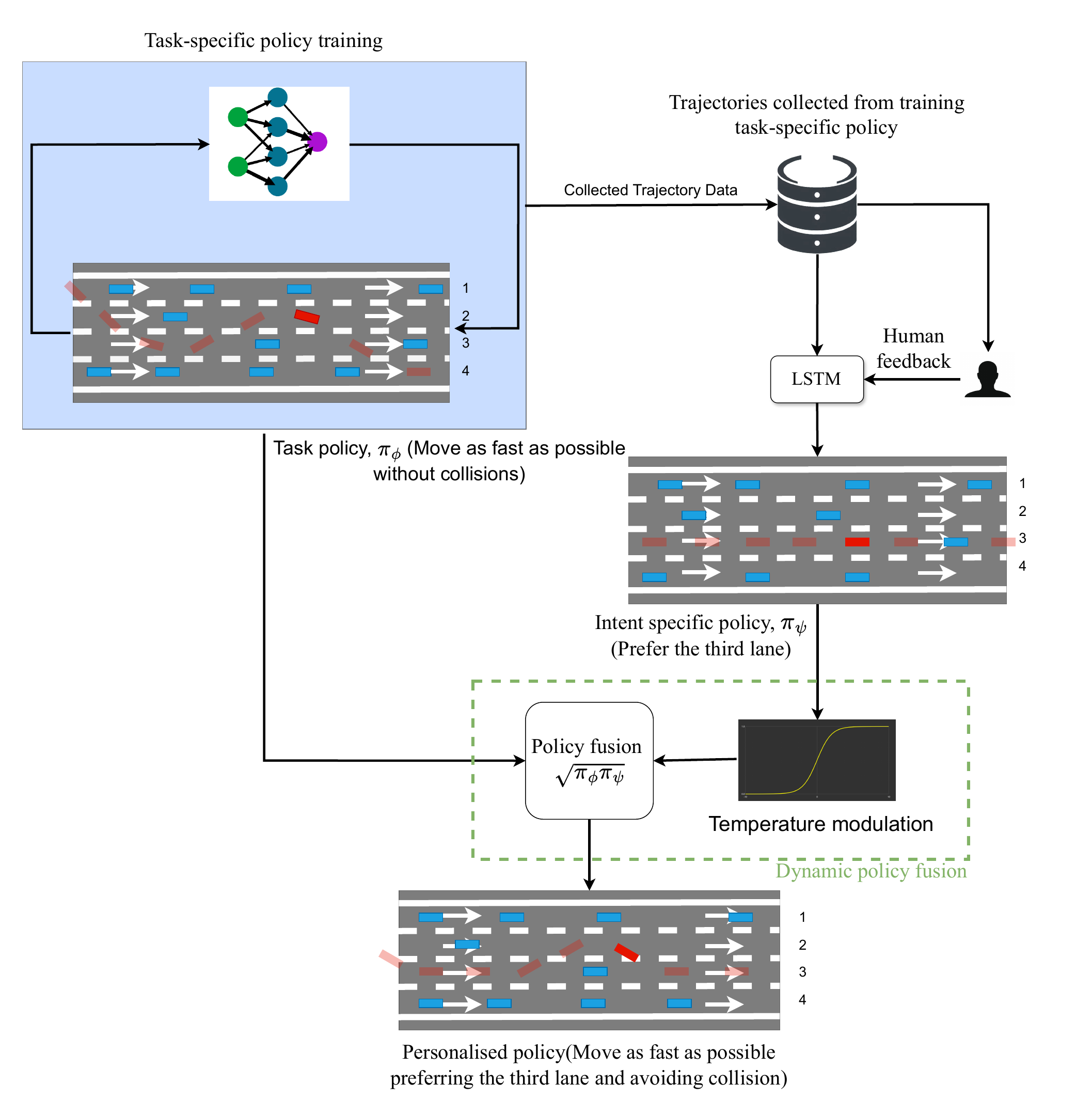}
\caption{Summary of our method. Trajectory-level training data used to train the task policy is labelled by a human user and an LSTM model is employed to identify the intent of the human. The output from the LSTM is converted to an intent-specific policy which is then dynamically fused by modulating the temperature parameter of the intent-specific policy (Discussed in Section \ref{sec:dynamic_policy_fusion}) with the task-specific policy to obtain the personalised policy. The blue shaded block is the task policy training and the green dotted block is the dynamic policy fusion. (Image taken from previous study \cite{ajsal2024personalisation})} \label{fig:overview}
\end{figure*}

\subsection{Learning Human Intent using LSTM}
\label{sec:intent_learning}

To learn the intent-specific policy, as previously described, we leverage the LSTM-based approach of RUDDER. We train this LSTM using human trajectory-level feedback on the same training data used to learn the task-specific policy. Hence, no additional interaction data is required. To get the human feedback, we simulated the human. The feedback score is the number of times the human preference is met within a trajectory (Refer to supplementary Section A for more details).

At each time step, the state-action vector is fed into the LSTM units. In the case of image inputs, we pre-train a Variational Autoencoder (VAE) to reduce the dimensionality of the state. We use Feature-wise Linear Modulation (FiLM) \citep{perez2018film} to modulate the state vector with the action vector. In \cite{perez2018film}, a FiLM network conditions the feature map of the neural network depending on another input signal. Here, we condition the state feature with the action vector and train the LSTM using the human feedback data as training labels. With this training setup, (More details in supplementary Section B and C) the model approximates the Q-value, which is then converted to the intent-specific policy as described in the next section. 
\subsection{Policy Construction}
\label{sec:policy_constrcution}
We use a Deep Q Network (DQN) parameterised by $\phi$ to learn the task and it produces the Q-values ($Q$) for each state-action pair. Similarly, the LSTM parameterised by $\psi$ also produces the Q-values ($Q'$) corresponding to the human intent. We choose the Boltzmann distribution to convert the Q-values into a policy. Therefore $\pi_\phi$ is the task-specific policy and $\pi_\psi$ is the intent-specific policy. This choice is motivated by the fact that if the intersection of the support of two policies is an empty set, their product would result in a random policy. By contrast, the Boltzmann distribution assigns at least a small probability to all actions, ensuring that the support of the distribution covers the entire action space. 
\begin{equation}
\label{eq:t1}
    \pi_\phi(a|s_t) = \frac{\exp{(\frac{Q(s_t,a)}{T_{\phi}})}}{\sum_{a \in \mathcal{A}} \exp{(\frac{Q(s_t,a)}{T_{\phi}})}},
\end{equation}

\begin{equation}
\pi_\psi(a|s_t) = \frac{\exp{(\frac{Q'(s_t,a))}{T_{\psi}})}}{\sum_{a \in \mathcal{A}} \exp{(\frac{Q'(s_t,a))}{T_{\psi}})}},
\label{eq:intent_policy_boltzmann}
\end{equation}
where $T_{\phi}$ is the temperature corresponding to task-specific policy and $T_{\psi}$ is the temperature corresponding to intent-specific policy. With policies constructed from Q-values, these constituent policies are fused together to produce the \emph{personalised policy}.

\subsection{Policy Fusion}
\label{sec:fusion}
Policy fusion is a process of combining two or more policies to produce a new policy. ~\cite{sestini2021policy} discussed several approaches through which policy fusion can be performed. We also aim to obtain a personalised policy by fusing the intent-specific and task-specific policies. However, policy fusion in the context of personalisation should satisfy two constraints. Firstly, fused policy should be identical to the task-specific policy if the task-specific and intent-specific policies are identical to each other - a property that we refer to as the invariability constraint. This constraint is crucial because it ensures that when the objectives of the task and the human are aligned, the personalised policy remains consistent and does not deviate during the fusion of identical constituent policies. Correspondingly, we define invariant policies as follows:
\begin{definition}[Invariant policies]
    Two policies $\pi_1(a|s)$ and $\pi_2(a|s)$ are said to be invariant policies if 
    \[
    KL\!\left( \pi_1(a|s) \,\middle\|\, \pi_2(a|s) \right) = 0 
    \quad \forall \, s \in \mathcal{S}.
    \]
\end{definition}
The second constraint is that the fused policy should act on the common support of the policies being fused.
To accommodate the aforementioned constraints, we construct the personalised policy through a new fusion method as follows: 
\begin{equation}
\label{Eq:fusion}
    \pi_f(a|s) = \tfrac{1}{Z}\sqrt{\pi_\phi(a|s_t) \times \pi_\psi(a|s_t)},
\end{equation}
where Z is the normalising factor.
Since $\pi_{\phi}$ and $\pi_{\psi}$ are independent, their joint probability can be expressed as the product of $\pi_\phi$ and $\pi_\psi$. This ensures that the fused policy acts on the common support. Moreover, this fusion method ensures both the task-specific policy and the fused policy are invariant policies if the task-specific policy and the intent-specific policy are identical, thereby satisfying the invariability constraint.
\begin{lemma}
    If $\pi_\phi(a|s) = \pi_\psi(a|s)$, then $\pi_\phi(a|s)$ and $\pi_f(a|s)$ are invariant policies.
\end{lemma}

\begin{proof}
    By assumption, $\pi_\phi(a|s) = \pi_\psi(a|s)$. Hence
    \[
        KL\!\left(\pi_\phi(a|s) \,\middle\|\, \pi_f(a|s)\right) 
        = KL\!\left(\pi_\phi(a|s) \,\middle\|\, \pi_\phi(a|s)\right) = 0.
    \]
    Therefore, $\pi_\phi(a|s)$ and $\pi_f(a|s)$ are invariant policies.
\end{proof}

The above lemma captures the invariability of the proposed fused method when $\pi_\phi(a|s) = \pi_\psi(a|s)$. However, if the deviation between the Q-values of the intent-specific policy and the task-specific policy is bounded (given the deviation of the temperature of both policies is bounded), then the divergence between the task-specific policy and the personalised policy can be bounded as stated in Theorem \ref{the:bound}. This theorem constraints the proposed fusion method to arbitrarily deviate from the task policy.

\begin{theorem}
\label{the:bound}
    Let $Q$ and $Q'$ be Q-values corresponding to the task policy and the intent-specific policy respectively. Let $\pi_\psi(a|s)$ and $\pi_\phi(a|s)$ represent the respective policies, with corresponding temperatures $T_\psi$ and $T_\phi$. Let $\norm{Q(s,a) - Q'(s,a)}_2 < \epsilon \ \forall s\in \mathcal{S} \ and \ a\in\mathcal{A}$  and $\norm{T_\psi-T_\phi}_2 < \delta$. Then, 
    \begin{center}
        \(KL\left(\pi_\phi(a|s) \middle\| \pi_f(a|s)\right) \leq\log \left( Z\right) + \dfrac{1}{2}
            \left(\dfrac{Q^{*}\delta + \epsilon T_\phi}{T_\phi T_\psi} \right)
            +\dfrac{1}{2}\log\left(\zeta\right) \ \forall a\in\mathcal{A}, s\in\mathcal{S}\),
    \end{center}
    where $Z$ is the normalising factor of $\pi_f$, $Q^{*} = max_{a \in \mathcal{A}}Q(s, a)$ and $\zeta = \dfrac{h(Q',T_\psi)}{h(Q,T_\phi)}$, here $h(Q,T) = \sum_aexp\left(\dfrac{Q}{T}\right)$.
\end{theorem}
\begin{proof}
        Please see Section ~\ref{sec:proofs}
    \end{proof}

The above theorem establishes the boundedness of the divergence between the task-specific policy and the personalised policy. Nevertheless, we can consider other choices for policy fusion as studied in the previous work \citep{sestini2021policy}, such as the weighted average or the product of the two policies. However, the former does not operate on the common support, and the latter does not satisfy the invariability constraint, even when the policies are exactly the same, as evidenced in the Lemma \ref{lm:invariability}. For a detailed analysis, we refer the reader to the supplementary Section ~\ref{sec:other_choice}.

\begin{lemma}
\label{lm:invariability}
    Let $\pi_\psi$ and $\pi_\phi$ be intent-specific and task-specific policies respectively, and assume $\pi_\psi$ is not a random policy. Then $\pi_\phi$ and $\tfrac{1}{Z}\pi_\psi\pi_\phi$, where $Z$ is a normalising factor, are not invariant policies under any condition.
\end{lemma}

\begin{proof}
Please see Section ~\ref{sec:proofs}
\end{proof}

Despite the compliance to invariance constraints of the fused policy, in practice, we found that this approach of statically fusing policies is subject to certain pitfalls, as explained in the next section. Subsequently, we address the issue through our novel dynamic policy fusion approach in Section \ref{sec:dynamic_policy_fusion}.

\subsection{Pitfalls of Static Fusion}
\label{sec:pitfall}
With the static policy fusion technique described in Section \ref{sec:fusion}, a potential challenge arises wherein one of the component policies over-dominates the other. Consequently, the agent may disproportionately exhibit the corresponding behaviour, resulting in noncompliance with either the task-specific policy or the intent-specific policy. 

To illustrate this phenomenon, we consider a 2D Navigation scenario where an agent is tasked with reaching a target location. However, a human may wish for some checkpoint state, different from the target state, to be visited before the agent reaches its target. In this case, the task-specific policy corresponds to the actions along the shortest path towards the target, and an intent-specific policy would correspond to one inferred from human feedback, which exclusively favours visiting the checkpoint state. A static policy fusion approach in this case could lead to over-dominance of one of the component policies. For example, over-dominance of the intent-specific policy would lead to the agent visiting the checkpoint state indefinitely. This motivates the need for fusing the policies in a dynamic fashion, such that the agent respects the human's preferences, while also simultaneously completing the task at hand. We develop such a dynamic fusion technique to control the relative dominance of the individual policies by modulating the temperature parameter $T_{\psi}$ of the intent-specific policy. We now describe the details of our dynamic fusion strategy.

\begin{figure}
\centering
    \includegraphics[width=0.45\textwidth]{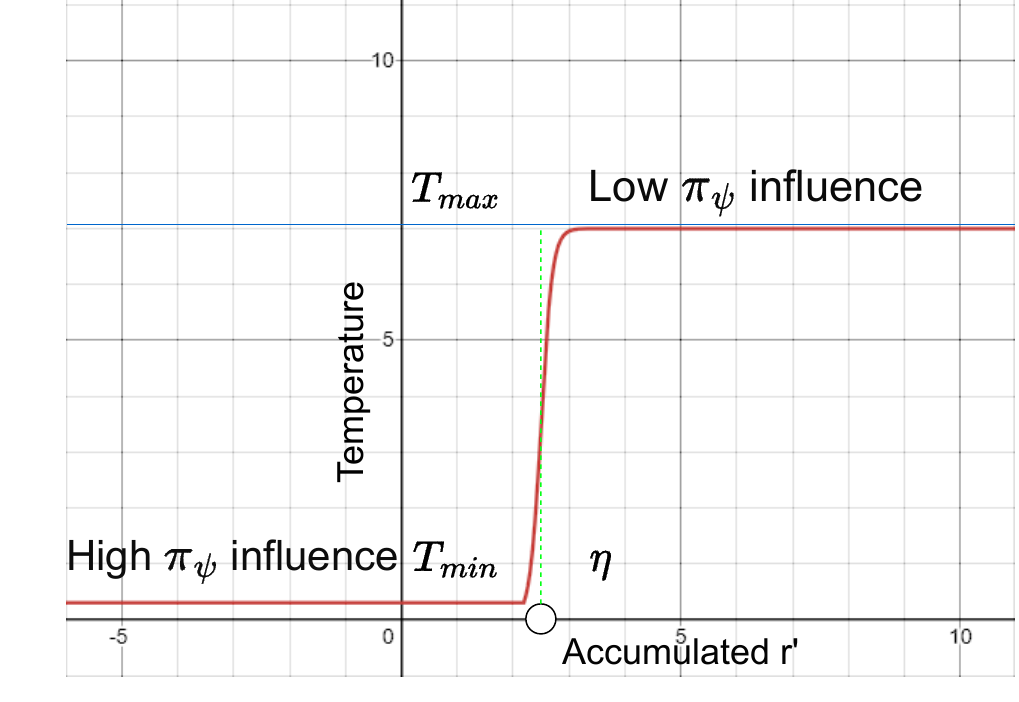}
    \caption{Plot showing the variation of $T_{\psi}$. The temperature rises when the accumulated $r'$ reaches $\eta$.}
    \label{fig:sigmoid}
\end{figure}

\setlength{\textfloatsep}{1pt}
\begin{algorithm}[t]
\caption{Personalising the action selection in an episode}
\label{alg:algorithm}
\begin{flushleft}
\textbf{Input}: DQN and LSTM Networks $\phi$ and $\psi$, accumulated reward threshold $\eta$, Task-specific policy temperature $T_{\phi}$, Min and Max temperature $T_{min}$ and $T_{max}$, slope of the sigmoid $m$ \\
\textbf{Reset} the environment to get the initial state $s_0$ \\
\textbf{Initialise}: $Q'(s_0,a_0) = 0$, $g(0) = 0$, $T_{\psi} = max(T_{min}, \frac{T_{max}}{1+exp(m\times\eta))})$, $done=False$, $t=0$
\end{flushleft}
\begin{algorithmic}[1] 
\While{not $done$}
\State $Q_{t} \gets []$ \Comment{DQN Q-values reset}
\State $Q'_{t} \gets []$ \Comment{LSTM Q-values reset}
\State $r_{t} \gets []$ \Comment{Human-induced reward reset}
\State $Q_{t} \gets \phi(s_t)$ \Comment{Invoking DQN}
\State $Q'_{t} \gets \psi(s_t,a)$ \Comment{Invoking LSTM}
\State Update $r_{t}$ as in Equation \eqref{eq:reward_redistribution}
\State $\pi_\phi \gets BoltzmanDistribution(Q_{t}, T_{\phi})$ 
\State $\pi_\psi \gets BoltzmanDistribution(Q'_{t}, T_{\psi})$ 
\State $a_t \gets \arg\max_{a \in \mathcal{A}} \sqrt{\pi_\phi(a|s_t) \times \pi_\psi(a|s_t)} $ 
\State $r' \gets r_{t} - mean(r_{t})$ \Comment{Eq \eqref{eq:adjust}}
\State $g(t) \gets \sum_{t'=0}^tr'(s_{t'},a_{t'})$
\State $T_{\psi} \gets max(T_{min}, \frac{T_{max}}{1+exp(-m(g-\eta))})$ \Comment{Eq \eqref{eq:temp}}
\State Execute the action $a_t$ from state $s_t$ to transition to $s_{t+1}$
\State $t \gets t+1$
\If {Episode terminate}
\State $done = True$.
\EndIf
\EndWhile
\end{algorithmic}
\end{algorithm}

\subsection{Dynamically Modulating Policy Dominance}
\label{sec:dynamic_policy_fusion}
To mitigate unintended policy dominance, we adopt the idea that when the temperature parameter $T_{\psi}$ in Equation \eqref{eq:intent_policy_boltzmann} is increased, the probability distribution of actions tends to become uniform, thereby reducing the influence of the intent-specific policy on the personalised policy. We therefore modulate $T_{\psi}$ depending on whether the fused personalised policy exhibits over-adherence or under-adherence to the intent-specific policy $\pi_{\psi}$.

\noindent\textbf{When to increase $T_{\psi}$?}
$T_{\psi}$ should increase if the intent-specific policy $\pi_{\psi}$ is enforced too strongly, which is characterised by high accumulated human-induced rewards. We design $T_{\psi}$ to increase when the accumulated human-induced reward surpasses an \emph{accumulated reward threshold} $\eta$.

However, we note that human intent can be specified in various modes: \emph{preference} (where a state is preferred over others), \emph{avoidance} (where the preference is to avoid a particular state) or \emph{mixed} (where the preference is to avoid certain states and to prefer certain others). In avoidance cases, human feedback assigns lower scores to unfavourable trajectories (those that pass through the state(s) to be avoided)  while in the preference case, higher scores are assigned to preferred trajectories (those that pass through the preferred state(s)). Depending on the trajectory score from the human, the LSTM assigns rewards to each state-action pair as in Equation \eqref{eq:reward_redistribution} which we refer to as \emph{human-induced rewards}.

To modulate $T_{\psi}$, we first compute the shifted rewards as follows:
\begin{equation}
\label{eq:adjust}
r' = r- mean(r),
\end{equation}
where $r$ is the vector that contains the rewards for different actions from a given state. The purpose of this shifting operation is to ensure that irrespective of the mode of human intent, the reward vector $r'$ contains elements with both positive as well as negative values.  This shifted human-induced reward $r'$ is then used to modulate $T_\psi$ as: 
\begin{equation}
\label{eq:temp}
T_{\psi}(t) = \max \left(T_{min}, \frac{T_{max}}{1+\exp\left(-\left(g(t)-\eta\right)\right)}\right),
\end{equation}
where 
$T_{min}$ and $T_{max}$ are the minimum and maximum allowable temperatures and $g(t)=\sum_{t'=0}^{t}r'(s_{t'},a_{t'})$ is the accumulated shifted human-induced reward. A plot of $T_\psi$ is shown in the Fig \ref{fig:sigmoid}.

\begin{figure*}[h]
\centering
    \includegraphics[width=1\linewidth]{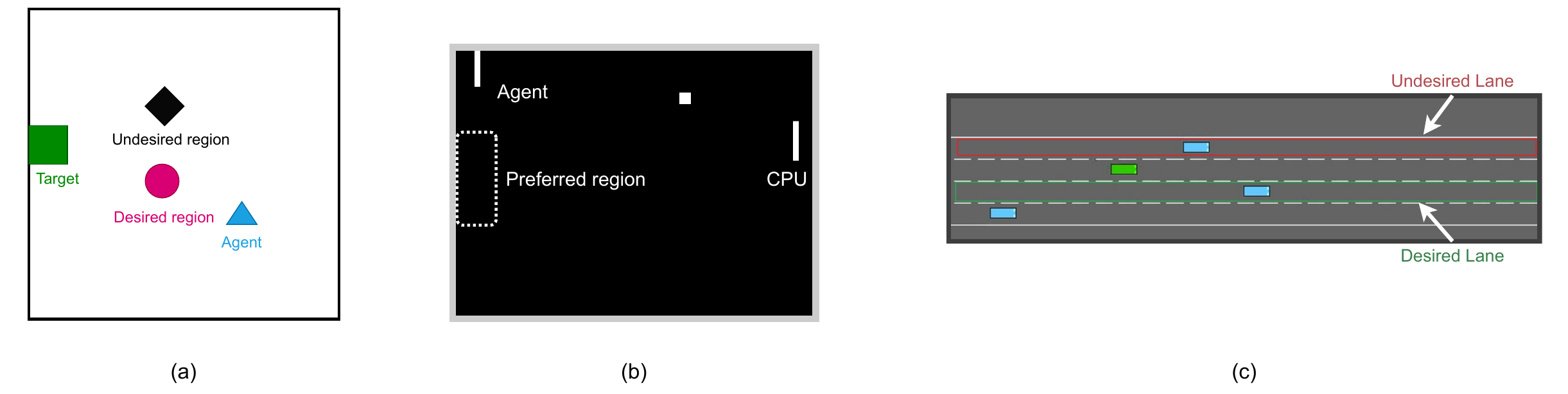}
    \caption{Snippets from the 2D environment (a), Pong (b) and Highway (c). In 2D navigation, the observation is converted to a grey-scale image. In Pong (b), the right paddle is controlled by the CPU and the left is controlled by the agent. The preferred region corresponds to the centre, as indicated. In the Highway environment (c), lanes marked with green and red boxes represent desired and undesired lanes respectively.}
    \label{fig:2d_pool}
\end{figure*}

\noindent\textbf{How does shifting (Equation \eqref{eq:adjust}) enable two-way switching?}
A high accumulated shifted human-induced reward ($>\eta$) indicates the agent has obeyed the intent-specific policy $\pi_{\psi}$ disproportionately more, which flags the need to mitigate the influence of $\pi_{\psi}$. Since the temperature in Equation \eqref{eq:temp} is monotonic with respect to accumulated reward $r'$, a higher accumulated $r'$ outputs a higher temperature, and this reduces the influence of the $\pi_{\psi}$ in the policy fusion step. Conversely, if the accumulated reward $r'$ $<\eta$, the temperature is lowered, which strengthens $\pi_{\psi}$. Therefore, this setup enables a two-way dynamic switching between the policies. Based on the threshold $\eta$, the influence of $\pi_{\psi}$ is reduced when the accumulated $r'$ is high, and strengthened when it is too low.
Our overall algorithm for personalisation is summarised in Algorithm \ref{alg:algorithm}.

\begin{table*}[ht]
\caption{\label{tab:personalisation} Different personalisation modes and their corresponding human and task objectives.}
\centering
\begin{tabular}{l l l l}
\toprule
\textbf{Environment} & \textbf{Mode} & \textbf{Human objective} & \textbf{Task objective} \\
\midrule

\multirow{3}{*}{2D Navigation} 
& Preference & Prefer a region & Move to the target \\
& Avoidance  & Avoid a region  & Move to the target \\
& Mixed      & Prefer a region and avoid another & Move to the target \\
\midrule

Pong 
& Preference & Prefer certain paddle positions & Win the game \\
\midrule

\multirow{3}{*}{Highway} 
& Preference & Prefer a lane & Move as fast as possible without colliding \\
& Avoidance  & Avoid a lane  & Move as fast as possible without colliding \\
& Mixed      & Prefer a lane and avoid another & Move as fast as possible without colliding \\
\bottomrule
\end{tabular}
\end{table*}

\section{Experiments}
\label{sec:experiments}
We demonstrate the effectiveness of our approach in our custom 2D Navigation environment,  Highway \citep{highwayenv}, Pong \citep{tasfi2016PLE} depicted in Fig \ref{fig:2d_pool}. Specifically, we conduct personalisation experiments in three scenarios: \emph{Avoidance}, \emph{Preference} and \emph{Mixed} \footnote{Code for our implementation can be found \href{https://github.com/Ajsal-Shereef/Personalization}{here}}. For Pong, we only consider the \emph{Preference} scenario. To learn the intent-specific policy, we sampled 2000 trajectories from saved trajectory data. Now we describe the environments as follows:

\begin{table*}[ht]
\caption{Results in 2D Navigation. Undesired/Desired region is the average visit out of a total of 20 time steps. Values marked in \textcolor{blue}{blue} indicate a top 15\% task performance score relative to the best baseline, and those in \textcolor{red}{red} indicate a top 15\% measure relative to the baseline that best adheres to human preferences. $\uparrow$ indicates higher is better and $\downarrow$ indicates lower is better. Results are averaged over 10 seeds, with each seed containing 300 episodes. The hyperparameters used are $T_{\phi} = 0.4$, $T_{\min} = 1$, $T_{\max} = 10$, and $\eta = 0$.}
\centering
\label{tab:2d_static}
\begin{tabular}{l l c c c}
\toprule
Mode & Method & \begin{tabular}{@{}c@{}}Desired region \\ visits $\uparrow$\end{tabular} & \begin{tabular}{@{}c@{}}Undesired region \\ visits $\downarrow$\end{tabular} & \begin{tabular}{@{}c@{}}Score \\ $\uparrow$\end{tabular} \\
\midrule

\multirow{5}{*}{Preference} 
& DQN             & $0.085 \pm 0.005$            & -- & \textcolor{blue}{$1.000 \pm 0.000$} \\
& RUDDER          & \textcolor{red}{$6.546 \pm 0.342$} & -- & $0.006 \pm 0.005$ \\
& MORL            & \textcolor{red}{$9.575 \pm 0.021$} & -- & $0.000 \pm 0.000$ \\
& Static          & $3.788 \pm 0.452$            & -- & $0.602 \pm 0.048$ \\
& Dynamic (Ours)  & $1.459 \pm 0.140$            & -- & \textcolor{blue}{$1.000 \pm 0.000$} \\
\midrule

\multirow{5}{*}{Mixed} 
& DQN             & $0.050 \pm 0.003$            & $0.183 \pm 0.006$ & \textcolor{blue}{$1.000 \pm 0.000$} \\
& RUDDER          & \textcolor{red}{$2.790 \pm 0.311$} & \textcolor{red}{$0.000 \pm 0.000$} & $0.007 \pm 0.007$ \\
& MORL            & $0.192 \pm 0.022$            & $0.119 \pm 0.016$ & $0.396 \pm 0.035$ \\
& Static          & $1.116 \pm 0.318$            & \textcolor{red}{$0.000 \pm 0.000$} & \textcolor{blue}{$0.894 \pm 0.035$} \\
& Dynamic (Ours)  & $0.268 \pm 0.058$            & \textcolor{red}{$0.000 \pm 0.000$} & \textcolor{blue}{$1.000 \pm 0.000$} \\
\midrule

\multirow{5}{*}{Avoidance} 
& DQN             & -- & $0.094 \pm 0.006$ & \textcolor{blue}{$1.000 \pm 0.000$} \\
& RUDDER          & -- & $0.000 \pm 0.000$ & $0.292 \pm 0.055$ \\
& MORL            & -- & $0.026 \pm 0.021$ & \textcolor{blue}{$0.894 \pm 0.099$} \\
& Static          & -- & \textcolor{red}{$0.005 \pm 0.005$} & \textcolor{blue}{$1.000 \pm 0.000$} \\
& Dynamic (Ours)  & -- & \textcolor{red}{$0.006 \pm 0.005$} & \textcolor{blue}{$1.000 \pm 0.000$} \\
\bottomrule
\end{tabular}
\end{table*}


\noindent \textbf{2D Navigation:} This is our custom environment where an agent is tasked with navigating to a target state while avoiding undesired states and visiting desired ones. The observation consists of a $40 \times 40$ grayscale image with four directional movements allowed. Actions that put the agent out of the frame are invalid.  A +1 reward is granted upon reaching the target location, and a 0 reward is given for all other actions. The episode concludes within 20 timesteps or upon reaching the target.
 
 \noindent \textbf{Pong:} In Pong, the observations correspond to ball x (horizontal) and y (vertical) position and velocity, player y position and velocity and CPU y position. Actions involve increasing, decreasing, or changing the velocity of the paddle. We have changed the original environment reward scheme and the termination criteria as follows. A +1 reward is obtained upon winning, a 0 reward for losing the game and 0.5 if the trajectory terminates. Pong concludes after 1000 timesteps or upon game outcome (win or loss).

 \noindent \textbf{Highway:} The goal is to navigate through traffic as fast as possible. The observations consist of a 26-dimensional vector, horizontal (x) and vertical(y) coordinates, corresponding x and y-velocities of five nearby vehicles, and the current lane. The agent can change lanes, stay idle, or adjust speed (move faster or slower). Any action that takes the agent out of frame is void. This environment has a dense reward setting, with positive rewards granted for maximizing speed, and negative rewards incurred upon collision with other vehicles. These rewards are normalized between 0 and 1. Episode termination occurs after 50 timesteps or upon collision with other vehicles.

 \noindent \textbf{Baselines:} We compare our proposed method against 4 baselines. 
 \begin{itemize}
     \item \emph{DQN} \citep{mnih2013playing}: which learns a policy that focuses solely on the  task objective.
     \item \emph{RUDDER} \citep{arjona2019rudder}: which learns a policy focusing on just the human objective.
     \item \emph{Static}: Static fusion method described in Section \ref{sec:fusion} designed to achieve both objectives.
     \item \emph{MORL} \citep{natarajan2005dynamic}: A multi-objective RL baseline scalarises the vector reward of task and human-induced reward by linearisation.
 \end{itemize}

We refer the reader to the supplementary material for details of the hyperparameters used and baseline implementations.

Table \ref{tab:personalisation} summarises the different environments with corresponding human and task objectives. To personalise the behaviour, we assume human users have lane preferences or aversions in Highway. Likewise, in 2D Navigation, humans prefer the agent to avoid or visit certain regions. In Pong, personalisation is introduced by favouring specific paddle positions.

Although Section \ref{sec:methodology} discussed actions being sampled from policies, in practical applications, it is common to exploit greedy actions once the policy is learned. Hence, in all our experiments, we choose actions as $argmax_{a \in \mathcal{A}} \sqrt{\pi_\phi(a|s_t) \times \pi_\psi(a|s_t)}$. This is akin to using a low-temperature parameter setting in the fused Boltzmann policy.

\begin{table*}[h]
\caption{Results for the Highway environment. Undesired/Desired lane columns indicate the average number of visits in 50 time steps. Values marked in \textcolor{blue}{blue} indicate a top 15\% task performance score relative to the best baseline, and those in \textcolor{red}{red} indicate a top 15\% measure relative to the baseline that best adheres to human preferences. $\uparrow$ indicates higher values are better and $\downarrow$ indicates lower values are better for the quantity specified in the column. The hyper-parameters chosen are $T_{\phi} = 0.6, T_{\max} = 5, T_{\min} = 0.3, \eta = 0$.}
\centering
\label{tab:highway}
\begin{tabular}{l l c c c c}
\toprule
Mode & Method & \begin{tabular}{@{}c@{}}Desired lane \\ visits $\uparrow$\end{tabular} & \begin{tabular}{@{}c@{}}Undesired lane \\ visits $\downarrow$\end{tabular} & \begin{tabular}{@{}c@{}}Hits \\ $\downarrow$\end{tabular} & \begin{tabular}{@{}c@{}}Score \\ $\uparrow$\end{tabular} \\
\midrule

\multirow{4}{*}{Avoidance} 
& DQN             & --                 & $10.59 \pm 0.96$ & \textcolor{blue}{$0.11 \pm 0.01$} & \textcolor{blue}{$39.59 \pm 0.46$} \\
& RUDDER          & --                 & \textcolor{red}{$0.00 \pm 0.00$}  & $0.73 \pm 0.11$ & $18.43 \pm 2.54$ \\
& MORL            & --                 & $5.80 \pm 1.07$  & \textcolor{blue}{$0.14 \pm 0.23$} & \textcolor{blue}{$37.44 \pm 0.52$} \\
& Dynamic (Ours)  & --                 & \textcolor{red}{$0.19 \pm 0.11$}  & \textcolor{blue}{$0.09 \pm 0.02$} & \textcolor{blue}{$38.80 \pm 0.51$} \\
\midrule

\multirow{4}{*}{Preference} 
& DQN             & $10.9 \pm 0.49$   & --                & \textcolor{blue}{$0.07 \pm 0.02$} & \textcolor{blue}{$40.25 \pm 0.64$} \\
& RUDDER          & \textcolor{red}{$31.84 \pm 3.45$} & -- & $0.30 \pm 0.06$ & $28.45 \pm 1.65$ \\
& MORL            & $11.58 \pm 1.12$  & --                & \textcolor{blue}{$0.07 \pm 0.01$} & $38.14 \pm 0.30$ \\
& Dynamic (Ours)  & \textcolor{red}{$29.91 \pm 1.38$} & -- & \textcolor{blue}{$0.06 \pm 0.01$} & \textcolor{blue}{$39.27 \pm 0.59$} \\
\midrule

\multirow{4}{*}{Mixed} 
& DQN             & $11.83 \pm 0.64$  & $9.59 \pm 0.82$  & \textcolor{blue}{$0.08 \pm 0.02$} & \textcolor{blue}{$40.06 \pm 0.79$} \\
& RUDDER          & \textcolor{red}{$26.26 \pm 5.90$} & \textcolor{red}{$0.00 \pm 0.00$}  & $0.50 \pm 0.12$ & $24.67 \pm 3.00$ \\
& MORL            & $12.61 \pm 1.20$  & $12.36 \pm 1.58$ & \textcolor{blue}{$0.07 \pm 0.01$} & $36.97 \pm 0.59$ \\
& Dynamic (Ours)  & \textcolor{red}{$25.21 \pm 2.55$} & \textcolor{red}{$0.27 \pm 0.06$}  & \textcolor{blue}{$0.07 \pm 0.04$} & \textcolor{blue}{$39.15 \pm 0.49$} \\
\bottomrule
\end{tabular}
\end{table*}

\subsection{Results}

We begin by illustrating the limitations of static fusion (Section~\ref{sec:pitfall}) and comparing it with our proposed dynamic fusion in the 2D Navigation task.  

In this environment, the metric ``desired region visits'' quantifies adherence to user preference. Conceptually, the desired region (as specified by the human intent) may represent a checkpoint that the user expects the agent to acknowledge en route to the target. If the agent remains in this region indefinitely, it indicates over-dominance of the intent-specific policy, preventing progress toward the actual goal. An ideal personalised policy should therefore acknowledge the desired region but continue to complete the navigation task. As shown in Table~\ref{tab:2d_static}, our dynamic fusion achieves exactly this balance: the agent briefly visits the desired region to respect the user’s preference but then proceeds reliably toward the target. This stands in contrast to static fusion and RUDDER, both of which over-emphasise user intent and consequently fail to complete the task.  

The performance of static fusion is also highly sensitive to the temperature parameter $T_\psi$. We evaluated several values and found $T_\psi = T_{\max}/2$ to be the most effective across all three modes. As shown in Table~\ref{tab:2d_static}, static fusion performs reasonably well in the Avoidance mode, successfully avoiding undesired regions while maintaining strong task scores. However, in the Preference and Mixed modes, it exhibits degradation: the agent becomes trapped in the desired region, resulting in poor task performance. This behaviour again reflects the over-dominance of the intent-specific policy under static fusion.

\begin{table}[h]
\caption{\label{tab:pong} Results in the Pong environment with $\eta=0$. Values marked in \textcolor{blue}{blue} indicate a top 15\% task performance score relative to the best baseline, and those in \textcolor{red}{red} indicate a top 15\% measure relative to the baseline that best adheres to human preferences. \% Desired region is the average fraction of an episode (in percentage) that the agent stayed within the desired region shown in Fig~\ref{fig:2d_pool}(b).}
\centering
\begin{tabular}{l c c}
\toprule
Method & \begin{tabular}{@{}c@{}}Desired region \\ \% $\uparrow$\end{tabular} & \begin{tabular}{@{}c@{}}Score \\ $\uparrow$\end{tabular} \\
\midrule
DQN             & $8.00 \pm 0.63$   & \textcolor{blue}{$0.57 \pm 0.01$} \\
RUDDER          & \textcolor{red}{$89.99 \pm 0.14$} & $0.00 \pm 0.00$ \\
MORL            & $41.68 \pm 0.02$ & $0.00 \pm 0.00$ \\
Dynamic (Ours)  & $55.60 \pm 4.48$ & \textcolor{blue}{$0.48 \pm 0.02$} \\
\bottomrule
\end{tabular}
\end{table}

\begin{table*}[h]
\caption{Result when the user intent directly conflicts with the task objective in Highway in the preference case. The result is averaged over 5 seeds.}
\centering
\label{tab:highway_conflict}
\begin{tabular}{l c c c c} 
\toprule
Method & \begin{tabular}{@{}c@{}}Desired lane \\ \% $\uparrow$\end{tabular} & \begin{tabular}{@{}c@{}}Speed violations \\ \% $\downarrow$\end{tabular} & \begin{tabular}{@{}c@{}}Hits \\ $\downarrow$\end{tabular} & \begin{tabular}{@{}c@{}}Score \\ $\uparrow$\end{tabular} \\
\midrule
DQN & $18.2 \pm 4.24$ & $15.72 \pm 1.07$ & $0.04 \pm 0.02$ & $38.84 \pm 0.51$ \\ 
Dynamic (Ours) & $57.14 \pm 8.52$ & $0.00 \pm 0.00$ & $0.02 \pm 0.02$ & $34.92 \pm 0.70$ \\ 
\bottomrule
\end{tabular}
\end{table*}

\begin{table*}[h]
\caption{Result with binary feedback in Highway in the preference case. The result is averaged over 5 seeds.}
\centering
\label{tab:highway_binary}
\begin{tabular}{l c c c} 
\toprule
Method & \begin{tabular}{@{}c@{}}Desired lane \\ \% $\uparrow$\end{tabular} & \begin{tabular}{@{}c@{}}Hits \\ $\downarrow$\end{tabular} & \begin{tabular}{@{}c@{}}Score \\ $\uparrow$\end{tabular} \\
\midrule
DQN & $21.8 \pm 0.98$ & $0.07 \pm 0.02$ & $40.25 \pm 0.64$ \\
Dynamic-Binary feedback (Ours) & $60.78 \pm 13.66$ & $0.06 \pm 0.02$ & $38.28 \pm 0.58$ \\
\bottomrule
\end{tabular}
\end{table*}

\begin{table*}[h]
\caption{Result with binary feedback in Pong in the preference case. The result is averaged over 5 seeds.}
\centering
\label{tab:pong_binary}
\begin{tabular}{l c c c} 
\toprule
Method & \begin{tabular}{@{}c@{}}Desired region \% \\ \% $\uparrow$\end{tabular} & \begin{tabular}{@{}c@{}}Score \\ $\uparrow$\end{tabular} \\
\midrule
DQN             & $8.00 \pm 0.63$   & $0.57 \pm 0.01$ \\
Dynamic-Binary feedback (Ours) & $42.54 \pm 5.03$  & $0.54 \pm 0.05$ \\
\bottomrule
\end{tabular}
\end{table*}

The behaviour of the other baselines is consistent with expectations. DQN focuses solely on maximising the task score, while RUDDER overfits to human preferences at the expense of task success. MORL, which attempts to reconcile both objectives, is less effective than our method. In particular, MORL struggles in scenarios such as Mixed 2D Navigation, where conflicting objectives arise. By seeking a global compromise, MORL fails to respond to local preference requirements, whereas our dynamic fusion adapts on the fly, applying corrections when human adherence is necessary. Overall, our method achieves a robust trade-off, incorporating both task and preference signals across all modes of personalisation.

These findings highlight the limitations of static fusion and MORL, and establish the necessity of dynamic policy fusion to satisfy both human and task objectives simultaneously.

In the Highway environment (Table \ref{tab:highway}), the expected behaviours of the baselines are again evident: DQN prioritises task performance, achieving high scores and minimal collisions, while RUDDER strongly favours human intent, spending the majority of time in the preferred lane but incurring more collisions and lower task scores. MORL attempts to balance both but falls short, as seen from its relatively high number of undesired lane visits in Avoidance and Mixed modes, and fewer desired lane visits in Preference and Mixed modes.

By contrast, our dynamic fusion method consistently balances task success with user alignment. In Avoidance, it minimises collisions and rarely enters the undesired lane while maintaining scores close to DQN. In Preference, it spends significantly more time in the desired lane than DQN or MORL, without compromising speed or safety. In Mixed, it adheres to both the lane preference and avoidance objectives simultaneously while sustaining strong task performance. These results highlight that dynamic fusion enables local corrections—adjusting the influence of user preferences adaptively—rather than committing to a fixed global trade-off as in MORL.

Turning to the Pong environment (Table \ref{tab:pong}), we evaluate the Preference mode, where the user favours keeping the paddle within a central region. Here, DQN achieves the highest task score but spends little time in the preferred region, while RUDDER almost exclusively remains in the preferred region at the cost of completely failing the task. MORL performs somewhat in between, but still fails to achieve meaningful task performance. Our dynamic fusion method strikes a balance: the agent spends a substantial proportion of time in the preferred region while still maintaining competitive scores, demonstrating effective alignment with user intent without sacrificing overall task viability.

Together, these results across 2D Navigation, Highway, and Pong consistently confirm that dynamic fusion outperforms static fusion, RUDDER, and MORL by adaptively balancing task objectives with user intent. While occasionally incurring a small drop in optimal task score relative to DQN, our approach achieves meaningful personalisation, ensuring that agents respect user preferences while still completing the primary task. Visualisation of the trajectories are available \href{https://github.com/Ajsal-Shereef/Personalization}{here}.

We further evaluated the robustness of our approach under two challenging settings:

\textbf{What if user intent directly conflicts with the task objective?} In the Highway conflict scenario (Table~\ref{tab:highway_conflict}), the user preference was to remain in the third lane while not exceeding a specified speed limit, which directly opposes the task objective of maximising speed without collisions. As expected, DQN prioritises speed and achieves the highest task score but only partially adheres to the lane preference and often violates the speed constraint. In contrast, our dynamic fusion method successfully maintains the lane preference and respects the speed limit, while still avoiding collisions. Although this leads to a lower overall task score (due to reduced speed), the outcome is consistent with our theoretical guarantees (Theorem~\ref{the:bound}): when the user and task objectives diverge within a bounded range, the resulting fused policy will also deviate from the optimal task policy in a controlled, bounded manner.

\textbf{What if user feedback is binary?} We investigated the use of binary feedback to reduce annotation effort and cognitive load for users. Instead of providing graded scores, trajectories were simply labelled as 1 if they contained the preferred states for more than 25\% of the time, and 0 otherwise. As shown in Table ~\ref{tab:highway_binary} (Highway) and Table ~\ref{tab:pong_binary} (Pong), our dynamic fusion method remains effective even under this weaker supervision, producing behaviour that aligns with preferences while sustaining competitive task performance. This suggests that our approach is robust to coarse feedback and can remain practical in real-world scenarios where fine-grained labels are costly or noisy.

Together, these additional studies highlight the flexibility of dynamic fusion: it handles conflicting objectives with a controlled drop in performance and also performs reliably under binary human feedback.

\section{Limitations and Future Work}
\label{sec:limitation}
Even though our proposed dynamic policy fusion approach dynamically optimises two objectives, in its current state, it is restricted to discrete action spaces. Extension to continuous action domains could be explored in future work. Additionally, the choice of hyper-parameters is done using trial and error; reducing or learning the number of parameters will enhance the utility of the method. Our method does not capture the varying preferences of the human. To address this limitation, future research could explore adaptive methods, perhaps based on meta-RL \citep{beck2023survey} to adjust to changing human preferences with a few shot learning of the LSTM network. 
The experiments presented in this work also make use of a noise-free ground truth function to simulate human feedback. We acknowledge that human feedback may be subject to noise, which could affect the robustness of our presented approach.  Further research is necessary to adapt our approach to be more robust to noise, the performance of which could then be evaluated with noisy feedback from real humans. Another practical aspect to consider is that the number of trajectories one could expect a human to label is limited. From an initial investigation, our approach does exhibit a drop in performance as fewer labelled trajectories become available. Although this problem could partially be mitigated through the use of semi-supervised methods \citep{van2020survey} to automatically generate trajectory score labels, we recognise that more noise-tolerant and feedback-efficient variants of our method need to be explored. Nevertheless, we believe the proposed dynamic policy fusion method presents a novel and an in-principle, practical approach for adapting and personalising already trained policies.
\section{Conclusion}
\label{sec:conclusion}
We proposed a novel approach to personalise a trained policy to respect human-specified preferences using trajectory-level human feedback, without any additional interactions with the environment. Using an LSTM-based approach to infer human intent, we designed a theoretically grounded dynamic policy fusion method to ensure the resulting policy completes a given task while also respecting human preferences. We empirically evaluated our approach on the Highway, 2D Navigation and Pong environments and demonstrated that our approach is capable of handling various modes of intent while only minimally compromising the task performance. We believe our approach presents an elegant and scalable solution to the problem of personalising pretrained policies. With a growing focus on personalisation in applications such as chatbots, robotic assistants, self-driving vehicles, etc., we believe our approach has the potential for imminent and widespread impact. 

\section{Statements and Declarations}
\textbf{Funding}
The authors declare that no funds, grants, or other support were received during the preparation of this manuscript.\\
\textbf{Competing Interests}
The authors have no relevant financial or non-financial interests to disclose.\\
\textbf{Author Contributions}
All authors contributed to the study conception, brainstorming and formulating the idea. Material preparation, data collection, experimentation and analysis were performed by Ajsal Shereef Palattuparambil. The first draft of the manuscript was written by Ajsal Shereef Palattuparambil, and all authors commented on previous versions of the manuscript. All authors read and approved the final manuscript.\\
\textbf{Ethics Approval and Consent to Participate}
Research does not involve  Human Participants and/or Animals.\\
\textbf{Consent for Publication} Yes.\\
\textbf{data availability statement}
The code used for the experimentation and analysis is available \href{https://github.com/Ajsal-Shereef/Personalization}{here}.

\clearpage
\bibliography{sn-bibliography}
\clearpage
\begin{appendices}

\section{Simulating human feedback}
\label{sec:human_feedback}
Our method involves collecting human feedback for trajectories used for training the task-specific policy. This is done by simulating humans. We conducted experiments in three personalisation modes: preference, avoidance, and mixed. In each case, we counted the number of times the agent met the personalisation criteria within a trajectory, and this count is regarded as the ground truth score for that trajectory. For instance, in 2D navigation, if the human wishes the agent to visit a preferred region, we count the number of times the agent visited that state and a positive score is given to that trajectory. Similarly, in the avoidance case, a negative score is given, and in the mixed case, we sum both the positive and negative counts.

Applying this method in a real-world context may introduce additional considerations. A key practical aspect is how trajectories are presented to a human for labelling while upholding the claim of no additional environment interactions. This can be addressed by allowing the user to passively view replays (i.e., recordings) of trajectories that were already collected during the task policy’s training. Since the agent does not actively generate new state-action transitions, this replay process does not constitute additional interaction. 

In practice, the number of trajectories that can be labelled by a human may be limited, and labelling effort could be influenced by factors such as fatigue or subjective bias. While our experiments used a noise-free simulated human, future work could incorporate techniques to mitigate such issues, for example by employing semi-supervised learning, active learning, or noise-robust feedback aggregation methods. We view these extensions as promising directions to further enhance the practicality of our approach in real-world scenarios.

\section{Training of LSTM}
\label{sec:lstm_training}
As described in the main paper, we used the same trajectory training data of the task-specific policy to train the LSTM. We sampled 2000 trajectories along with the human feedback as the training data. 

The input to the LSTM is the state and the action chosen at time $t$ and it outputs the corresponding Q-values. We made a custom one-hot encoding of the action to nearly match the dimension of the state vector. For example, in the Highway environment, the state vector is 26 dimensional and there are 5 actions. We encode the selected action to a 25-dimensional vector with one filling in the corresponding 5 positions of the action and the rest being zero. We then condition the state vector with the action vector and train the LSTM. An overview of this dataflow is outlined in Fig \ref{fig:lstm_data}.

The LSTM network architecture is single-layered with 64 units. The output gate and forgot gate of the LSTM are set to 1 as in \cite{arjona2019rudder} as we don't want to forget anything while performing the credit assignment. The optimiser details are provided in Table \ref{tab:optimizer}.
We used the same loss functions as that of \cite{arjona2019rudder} to train the LSTM. 
\begin{eqnarray}
\label{eq:main}
L_m = (l-\Tilde{Q}_H)^2,
\end{eqnarray}
\begin{eqnarray}
L_c = \frac{1}{H+1}\sum_{t=0}^H(l-\Tilde{Q}_t)^2,
\end{eqnarray}
\begin{eqnarray}
L_e = \frac{1}{H-\delta+1}\sum_{t=0}^{H-\delta}(\Tilde{Q}_{t+\delta}-\beta)^2,
\end{eqnarray}
where $l$ is the human label and $\Tilde{Q}_t$ is the LSTM output at timestep $t$ and $H$ is the final timestep. $\beta$ is the prediction of $\Tilde{Q}_{t+\delta}$ at timestep $t$. $\delta$ is the forward lookup which is set to 3 in all our experiments.

\begin{figure}
    \centering
    \includegraphics[width=1\linewidth]{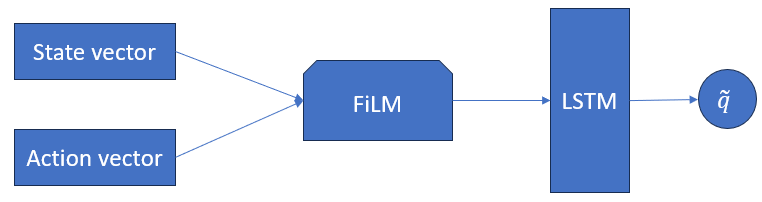}
    \caption{LSTM network training overview. The state vector and the action vector are fed to the FiLM network. The FiLM network conditions the action vector with a state vector and this conditioned vector is the input to the LSTM. LSTM outputs the estimated q-value.}
    \label{fig:lstm_data}
\end{figure}

The final loss function is,
\begin{eqnarray}
\label{eq:total loss}
L = L_m + \frac{1}{10}(L_c+L_e).
\end{eqnarray}
The $L_c$ term aims to allocate full credit to each time step without considering its impact on future steps. Basically, it constrains the expected future reward to zero. To get $\beta$ in $L_e$, we added a separate head to the LSTM network which predicts the Q-values at $t+\delta$ at time $t$. Both $L_e$ and $L_c$ act as auxiliary loss functions for reward redistribution.

\begin{table}[h]
\caption{\label{tab:optimizer} LSTM optimizer details.}
\centering
\begin{tabular}{l l}
\toprule
\textbf{Particulars} & \textbf{Value} \\
\midrule
Learning rate      & $10^{-4}$ \\
Weight decay       & $10^{-8}$ \\
Gradient clipping  & $10$ \\
Optimizer          & Adam \\
\bottomrule
\end{tabular}
\end{table}

\begin{table}[h]
\caption{\label{tab:optimizer_vae} VAE optimizer details.}
\centering
\begin{tabular}{l l}
\toprule
\textbf{Particulars} & \textbf{Value} \\
\midrule
Learning rate      & $10^{-3}$ \\
Weight decay       & $10^{-6}$ \\
Gradient clipping  & $10$ \\
Optimizer          & Adam \\
\bottomrule
\end{tabular}
\end{table}

\begin{table*}[h]
\caption{\label{tab:vae_encoder} VAE encoder architecture.}
\centering
\begin{tabular}{l c c c c c}
\toprule
\textbf{Conv layer} & \textbf{Input depth} & \textbf{Output depth} & \textbf{Kernel size} & \textbf{Stride} & \textbf{Padding} \\
\midrule
Conv 1 & 1   & 16  & 3 & 2 & 0 \\
Conv 2 & 16  & 32  & 3 & 2 & 0 \\
Conv 3 & 32  & 64  & 3 & 1 & 0 \\
Conv 4 & 64  & 128 & 3 & 1 & 0 \\
Conv 5 & 128 & 256 & 3 & 1 & 1 \\
\bottomrule
\end{tabular}
\end{table*}

\begin{table*}[h]
\caption{\label{tab:vae_decoder} VAE decoder architecture.}
\centering
\begin{tabular}{l c c c c c}
\toprule
\textbf{ConvT layer} & \textbf{Input depth} & \textbf{Output depth} & \textbf{Kernel size} & \textbf{Stride} & \textbf{Padding} \\
\midrule
ConvT 1 & 256 & 128 & 3 & 1 & 1 \\
ConvT 2 & 128 & 64  & 3 & 1 & 0 \\
ConvT 3 & 64  & 32  & 3 & 2 & 0 \\
ConvT 4 & 32  & 16  & 5 & 1 & 0 \\
ConvT 5 & 16  & 1   & 4 & 2 & 0 \\
\bottomrule
\end{tabular}
\end{table*}

Experiments were conducted across three distinct environments, each employing different modes of personalisation. For the 2D Navigation and Highway environments, we implemented experiments in avoidance, preference, and mixed modes. However, for the Pong environment, personalisation was conducted in preference mode. The LSTM loss function for all modes across the 2D Navigation, Highway, and Pong environments is depicted in Fig \ref{fig:2d_loss}, \ref{fig:highway_loss}, and \ref{fig:pong_loss}, respectively.

Given that the horizon for the Pong environment is 1000, which is relatively long, we opted to train the LSTM using snippets of the trajectory that contain 50 timesteps. We simulated the feedback for these short snippets. For all other environments, we utilized the full-length trajectory in our training process.

\begin{figure}
    \centering
    \begin{subfigure}{0.4\textwidth}
        \includegraphics[width=\textwidth]{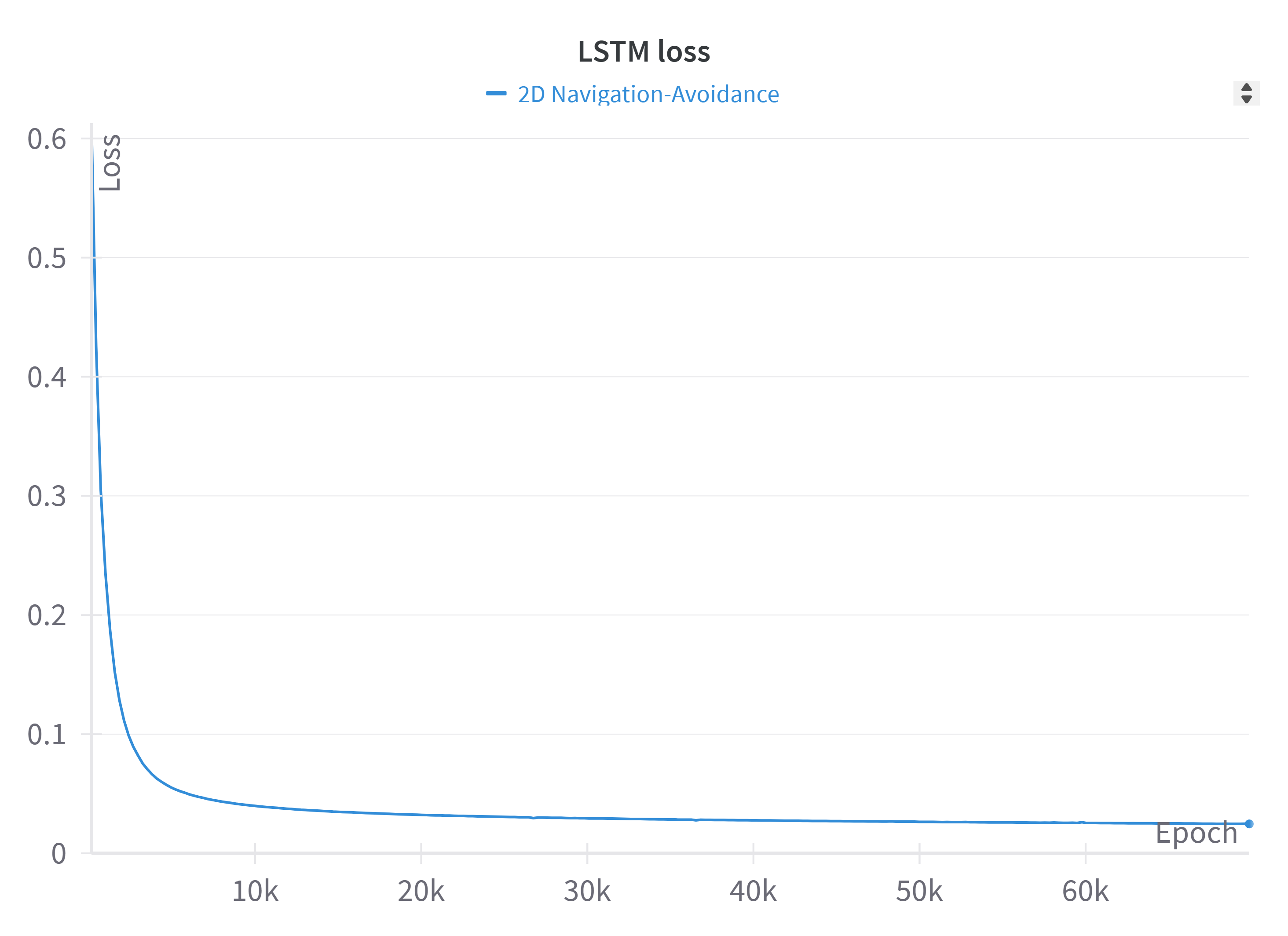}
        \caption{Avoidance}
    \end{subfigure}
    
    \vspace{1em} 

    \begin{subfigure}{0.4\textwidth}
        \includegraphics[width=\textwidth]{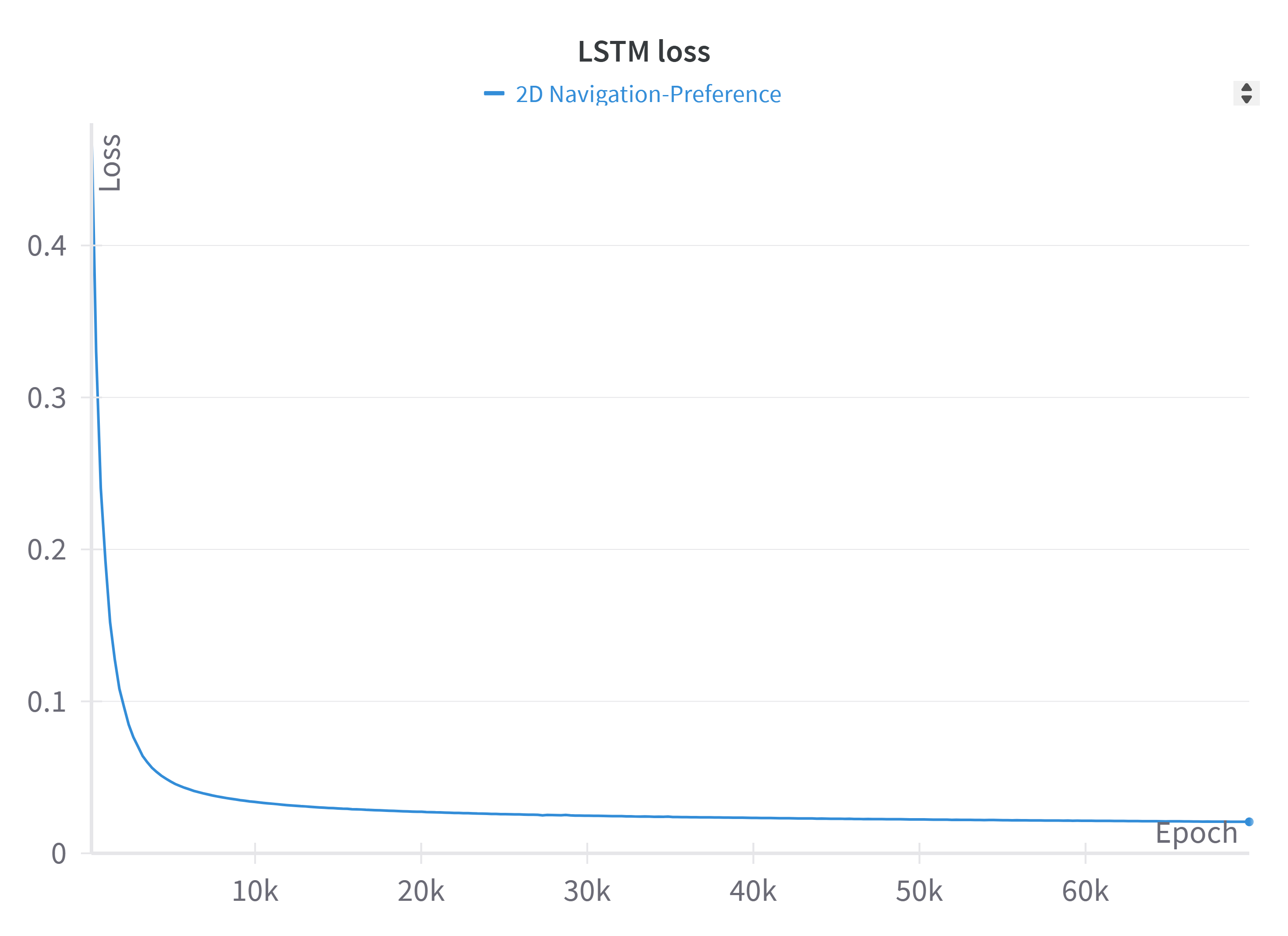}
        \caption{Preference}
    \end{subfigure}
    
    \vspace{1em} 

    \begin{subfigure}{0.4\textwidth}
        \includegraphics[width=\textwidth]{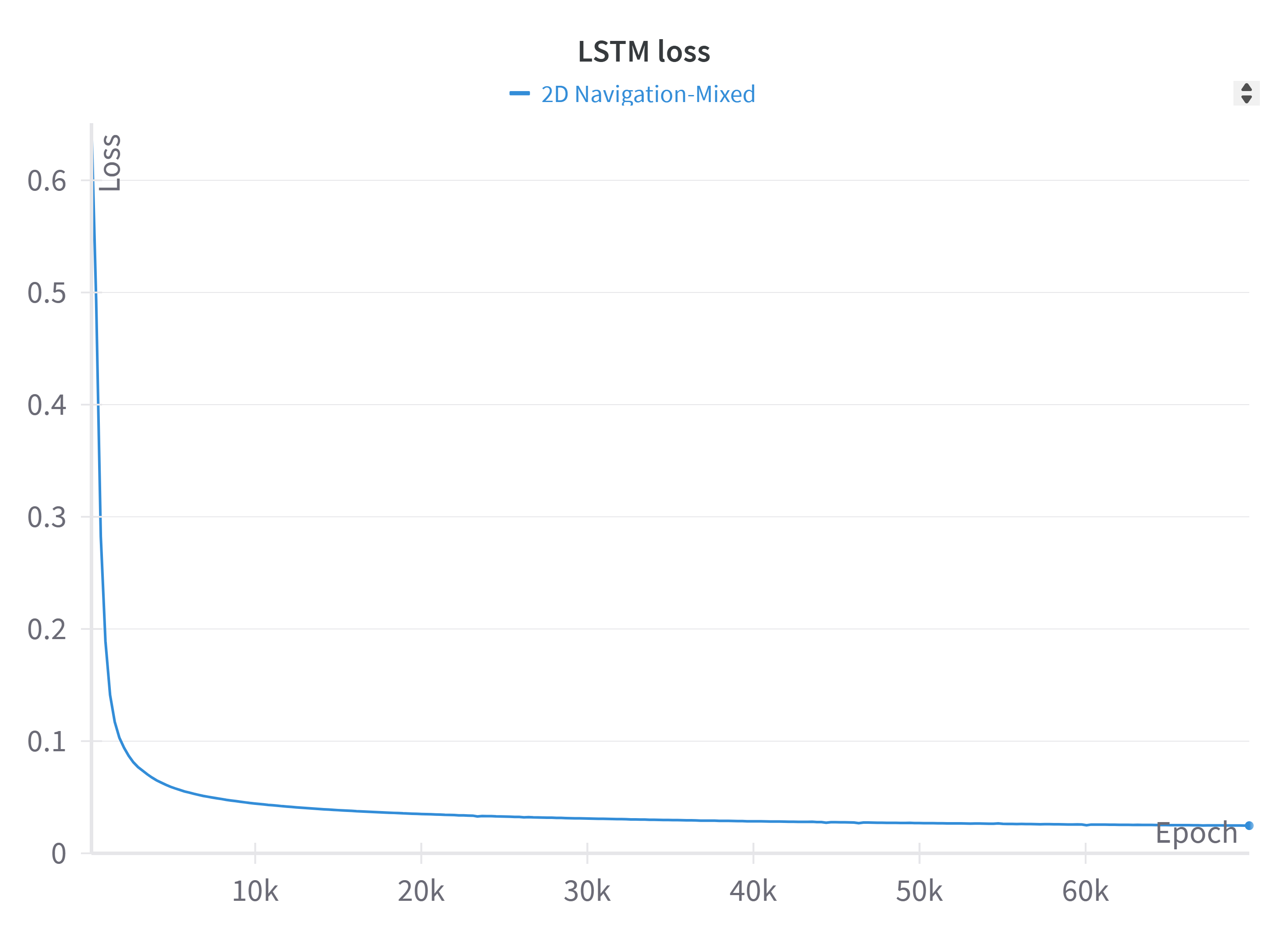}
        \caption{Mixed}
    \end{subfigure}
    \caption{LSTM loss function of 2D navigation across all modes of personalisation}
    \label{fig:2d_loss}
\end{figure}

\begin{figure}
    \centering
    \begin{subfigure}{0.4\textwidth}
        \includegraphics[width=\textwidth]{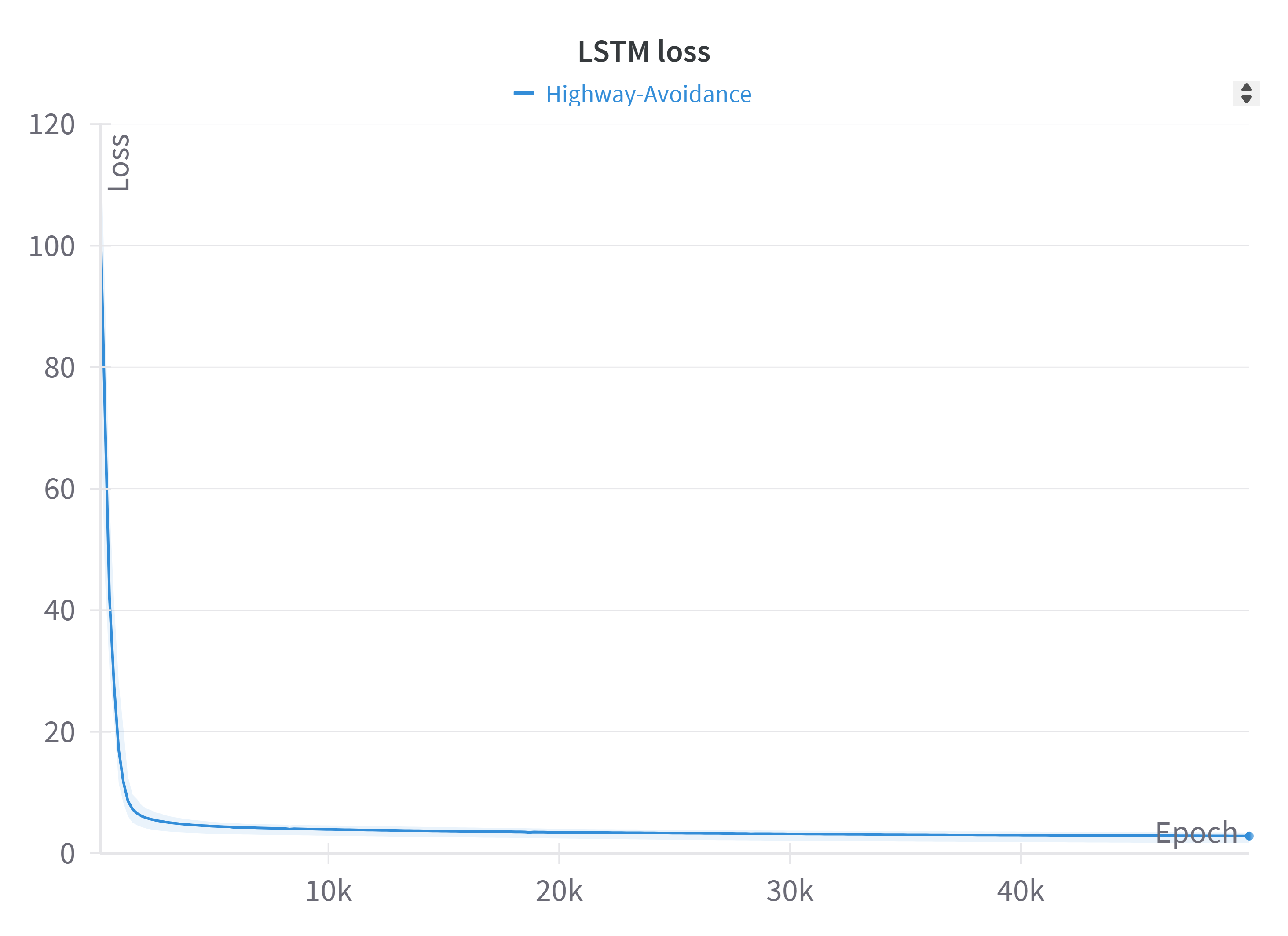}
        \caption{Avoidance}
        \label{fig:sub1}
    \end{subfigure}
    
    \vspace{1em} 

    \begin{subfigure}{0.4\textwidth}
        \includegraphics[width=\textwidth]{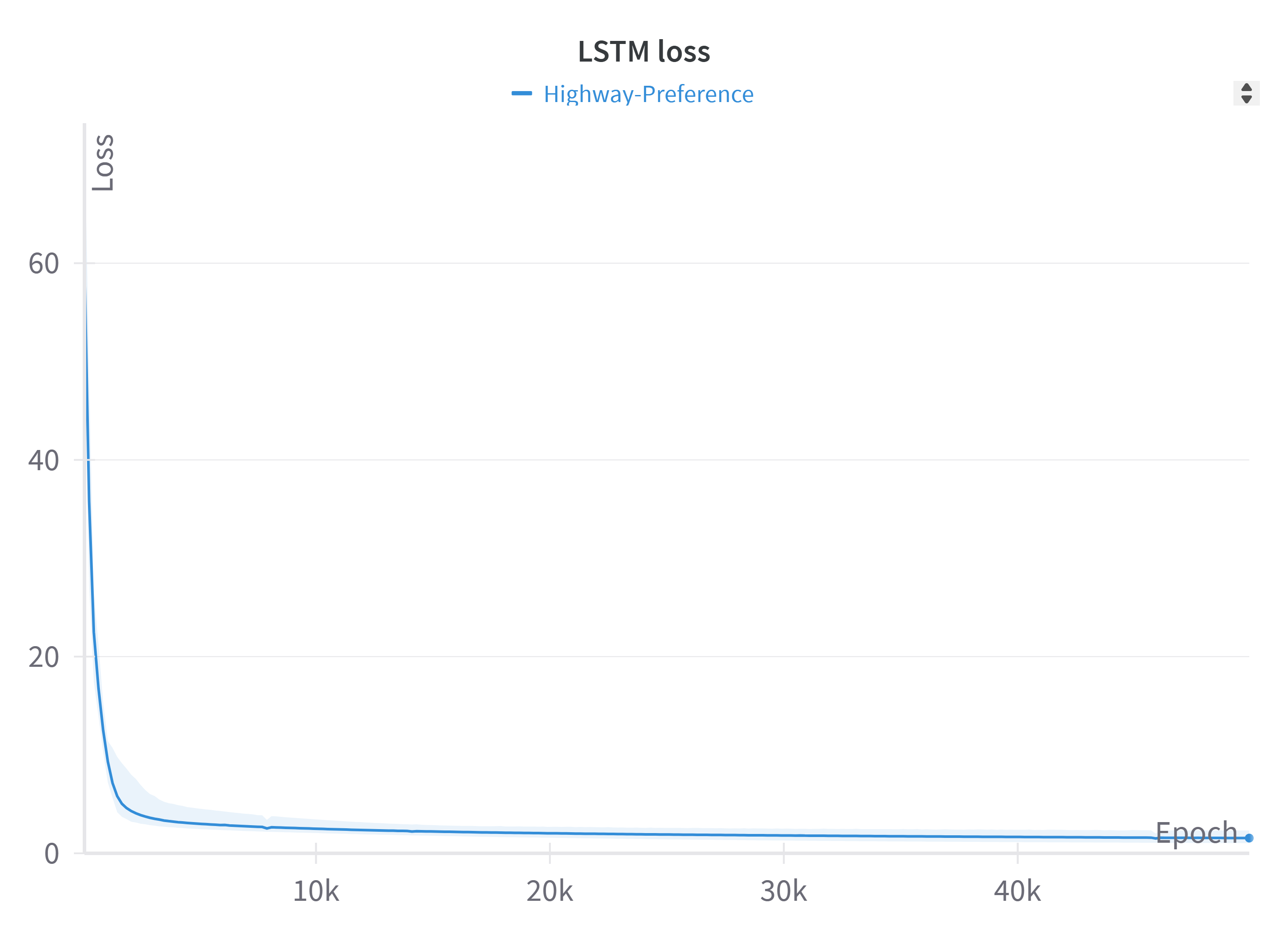}
        \caption{Preference}
        \label{fig:sub2}
    \end{subfigure}
    
    \vspace{1em} 

    \begin{subfigure}{0.4\textwidth}
        \includegraphics[width=\textwidth]{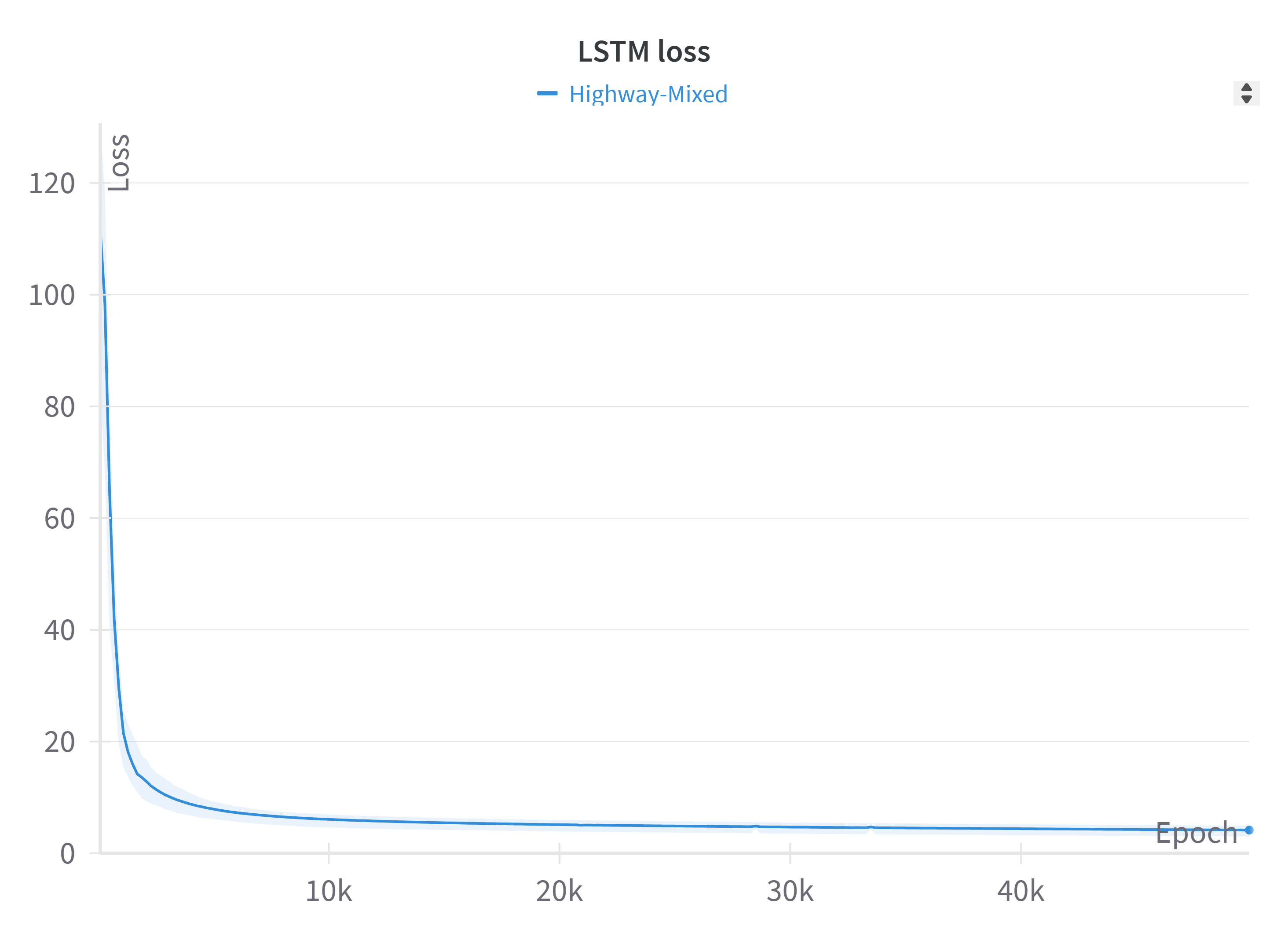}
        \caption{Mixed}
        \label{fig:sub3}
    \end{subfigure}
    \caption{LSTM loss function of Highway across all modes of personalisation}
    \label{fig:highway_loss}
\end{figure}

\begin{figure}
    \centering
    \includegraphics[width=0.3\textwidth]{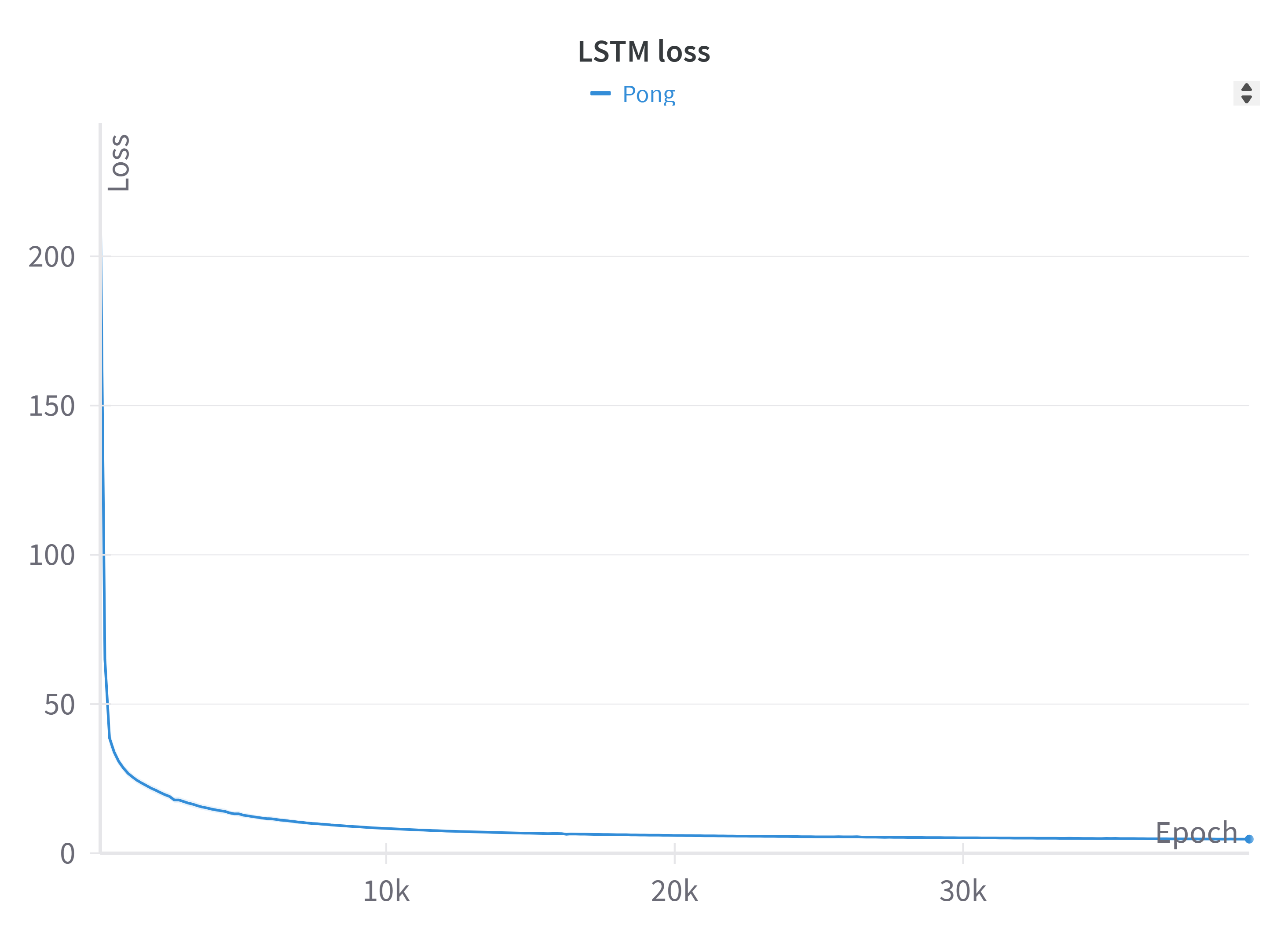}
    \caption{LSTM loss function of Pong}
    \label{fig:pong_loss}
\end{figure}

\begin{figure}
    \centering
    \includegraphics[width=0.3\textwidth]{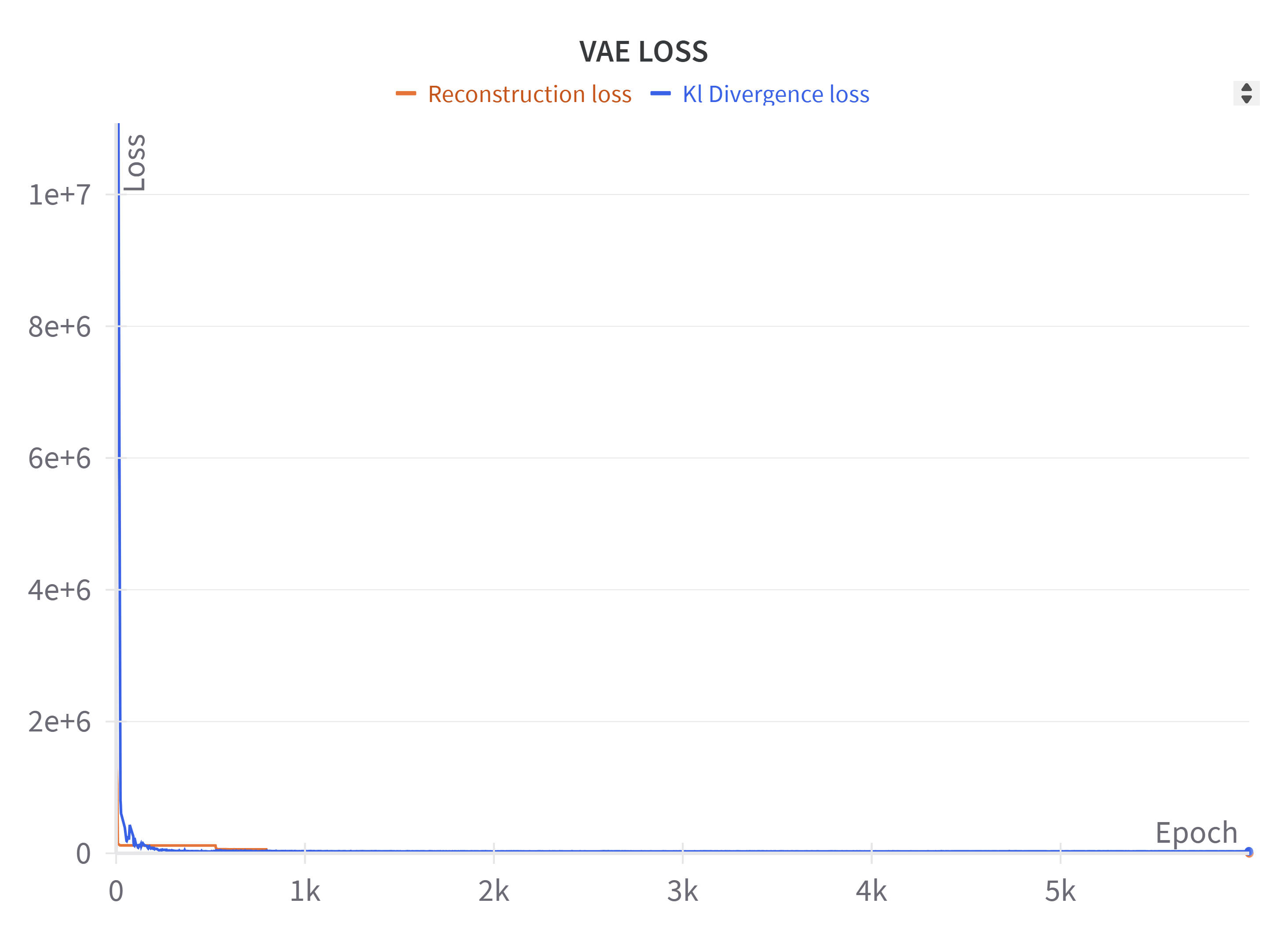}
    \caption{VAE loss function. The scale of the loss is high because the loss is summed over the data for each mini batch}
    \label{fig:vae_loss}
\end{figure}

\section{Training of VAE}
\label{sec:vae_training}
VAE is pretrained and only used if the observation of the environment is pixel-based. Training data is the same as the training data of the task policy. We use the conventional VAE loss functions which consist of reconstruction loss and KL-divergence loss. The encoder and decoder architectures are tabulated in Tables \ref{tab:vae_encoder}, \ref{tab:vae_decoder} respectively. A fully connected linear layer is used to map the output from the encoder to a dimension of 512 subsequently the output is mapped to a latent dimension (set as 16) using another linear layer. Optimiser details of the VAE is shown in Table \ref{tab:optimizer_vae}. The loss curve is shown in Fig \ref{fig:vae_loss}.

\section{Different choices of static fusion}
\label{sec:sestini_fusion}
Policy fusion is a process of combining two or more policies to produce a new policy. \cite{sestini2021policy} studied various policy fusion methods which are as follows.

\subsection{Product Policy}

The fused policy is obtained by multiplying different policies. The action selection process from the $\pi_f$ is as follows:

\begin{eqnarray}
\arg\max_{a \in \mathcal{A}} \pi_\phi(a|s_t) \times \pi_\psi(a|s_t)
\end{eqnarray}

\subsection{Mixture Policy}
In this fusion method, the fused policy is obtained by taking the average of the policies. The action selection process is expressed as follows:

\begin{eqnarray}
\arg\max_{a \in \mathcal{A}} \frac{\pi_\phi(a|s_t) + \pi_\psi(a|s_t)}{2}
\end{eqnarray}

\subsection{Entropy Threshold}
This method selects the action based on the entropy of each policy at state $s_t$. Let $\mathcal{H}{\phi}(s_t)$ represent the entropy of the task policy, and $\mathcal{H}{\psi}(s_t)$ represent the entropy of the Intent-specific policy at state $s_t$. The action is selected according to the following conditions:
\begin{eqnarray}
\begin{cases}
\arg\max_{a \in \mathcal{A}} \pi_\psi(a|s_t) & \text{If } \mathcal{H}{\psi}(s_t) < \mathcal{H}{\phi}(s_t) + \epsilon\\
\arg\max_{a \in \mathcal{A}} \pi_\phi(a|s_t) & Otherwise
\end{cases}
\end{eqnarray}
Here, $\epsilon$ is a small positive value.

\subsection{Entropy Weighted}
This fusion method computes a weighted average of the two policies, where the weights are determined by the minimum entropy among the two policies. Let $\mathcal{H}^{*} = \min\left(\mathcal{H}{\psi}(s_t), \mathcal{H}{\phi}(s_t)\right)$. The action is selected as follows:
\begin{eqnarray}
\arg\max_{a \in \mathcal{A}} \left(\mathcal{H}^{*} \times \pi_\phi(a|s_t) + (1 - \mathcal{H}^{*}) \times \pi_\psi(a|s_t)\right)
\end{eqnarray}
 The term $\mathcal{H}^{*} \times \pi_\phi(a|s_t)$ represents the weighted contribution of the task policy, and $(1 - \mathcal{H}^{*}) \times \pi_\psi(a|s_t)$ represents the weighted contribution of the Intent-specific policy.

\section{Proof of main theorem and lemma}
\label{sec:proofs}

\begin{theorem-restate}[\ref{the:bound}]
    Let $Q$ and $Q'$ be Q-values corresponding to the task policy and the intent-specific policy respectively. Let $\pi_\psi(a|s)$ and $\pi_\phi(a|s)$ represent the respective policies, with corresponding temperatures $T_\psi$ and $T_\phi$. Let $\norm{Q(s,a) - Q'(s,a)}_2 < \epsilon \ \forall s\in \mathcal{S} \ and \ a\in\mathcal{A}$  and $\norm{T_\psi-T_\phi}_2 < \delta$. Then, 
    \begin{center}
        \(KL\left(\pi_\phi(a|s) \middle\| \pi_f(a|s)\right) \leq\log \left( Z\right) + \dfrac{1}{2}
            \left(\dfrac{Q^{*}\delta + \epsilon T_\phi}{T_\phi T_\psi} \right)
            +\dfrac{1}{2}\log\left(\zeta\right) \ \forall a\in\mathcal{A}, s\in\mathcal{S}\),
    \end{center}
    where $Z$ is the normalising factor of $\pi_f$, $Q^{*} = max_{a \in \mathcal{A}}Q(s, a)$ and $\zeta = \dfrac{h(Q',T_\psi)}{h(Q,T_\phi)}$, here $h(Q,T) = \sum_aexp\left(\dfrac{Q}{T}\right)$.
\end{theorem-restate}
\begin{proof}
        \begin{align*}
            &KL\left(\pi_\phi \middle\| \pi_f\right)=
             \sum _{a}\pi _{\phi }\log \left( \dfrac{\pi_\phi Z}{\pi_f}\right) \\
            &=\log \left( Z\right) + \sum _{a}\pi_\phi
            \left(\log\left(\sqrt{e^{\dfrac{Q}{T_\phi}-\dfrac{Q'}{T_\psi}}}\right) +\log\left(\sqrt{\zeta}\right)\right) \\
            &=\log \left( Z\right) + \sum _{a}\dfrac{\pi_\phi}{2}
            \left(\dfrac{Q}{T_\phi}-\dfrac{Q'}{T_\psi} +\log\left(\zeta\right)\right) \\
            &=\log \left( Z\right) + \sum _{a}\dfrac{\pi_\phi}{2}
            \left(\dfrac{Q}{T_\phi}-\dfrac{Q'}{T_\psi} \right)
            +\dfrac{1}{2}\log\left(\zeta\right) \\
            &=\log \left( Z\right) + \sum _{a}\dfrac{\pi_\phi}{2}
            \left(\dfrac{QT_\psi-Q'T_\phi}{T_\phi T_\psi} \right)
            +\dfrac{1}{2}\log\left(\zeta\right) \\
            &=\log \left( Z\right) + \sum _{a}\dfrac{\pi_\phi}{2}
            \left(\dfrac{QT_\psi-Q'T_\phi +QT_\phi - QT_\phi}{T_\phi T_\psi} \right) \\
            &+\dfrac{1}{2}\log\left(\zeta\right) \\
            &\leq\log \left( Z\right) + \sum _{a}\dfrac{\pi_\phi}{2}
            \left(\dfrac{Q\delta + \epsilon T_\phi}{T_\phi T_\psi} \right)
            +\dfrac{1}{2}\log\left(\zeta\right)\\
            &\leq\log \left( Z\right) + \dfrac{1}{2}
            \left(\dfrac{Q^{*}\delta + \epsilon T_\phi}{T_\phi T_\psi} \right)
            +\dfrac{1}{2}\log\left(\zeta\right)
        \end{align*}
    \end{proof}

\begin{lemma-restate}[\ref{lm:invariability}]
    Let $\pi_\psi$ and $\pi_\phi$ be intent-specific and task-specific policies respectively, and assume $\pi_\psi$ is not a random policy. Then $\pi_\phi$ and $\tfrac{1}{Z}\pi_\psi\pi_\phi$, where $Z$ is a normalising factor, are not invariant policies under any condition.
\end{lemma-restate}

\begin{proof}
We compute the KL divergence between $\pi_\phi$ and $\tfrac{1}{Z}\pi_\psi\pi_\phi$:
\begin{align*}
    KL\!\left( \pi _{\phi }\;\middle\|\;\tfrac{\pi _{\phi }\pi _{\psi }}{Z}\right) 
    &= \sum _{a}\pi _{\phi}(a)\log \left( \frac{\pi_\phi(a)}{\tfrac{1}{Z}\pi _{\phi }(a)\pi _{\psi }(a)}\right) \\
    &= \sum _{a}\pi _{\phi }(a)\log \left( \frac{Z}{\pi _{\psi }(a)}\right) \\
    &= \log(Z) - \sum _{a}\pi _{\phi }(a)\log \pi _{\psi }(a),
\end{align*}
where the normalising factor is 
\[
Z=\sum _{a}\pi _{\phi }(a)\pi _{\psi }(a).
\]

For $\pi_\phi$ and $\tfrac{1}{Z}\pi_\psi\pi_\phi$ to be invariant, we require
\[
KL\!\left( \pi _{\phi }\;\middle\|\;\tfrac{\pi _{\phi }\pi _{\psi }}{Z}\right) = 0.
\]
That is,
\[
\log \left( Z\right) = \sum _{a}\pi _{\phi }(a)\log \pi _{\psi }(a).
\]
Equivalently,
\[
\log\!\left(\sum _{a}\pi _{\phi }(a)\pi _{\psi }(a)\right) 
= \sum _{a}\pi _{\phi }(a)\log \pi _{\psi }(a),
\]
which can be rewritten as
\[
\log \Big( \mathbb{E}_{\pi_\phi}[\pi_\psi]\Big) 
= \mathbb{E}_{\pi_\phi}[\log \pi_\psi].
\]

By Jensen's inequality, since $\log$ is concave, we have
\[
\mathbb{E}_{\pi_\phi}[\log \pi_\psi] \;\leq\; \log \Big(\mathbb{E}_{\pi_\phi}[\pi_\psi]\Big),
\]
with equality if and only if $\pi_\psi(a)$ is constant almost surely under $\pi_\phi$, i.e., if $\pi_\psi$ is a random (uniform) policy. This contradicts the assumption that $\pi_\psi$ is not a random policy.

Therefore, the KL divergence cannot be zero, and hence $\pi_\phi$ and $\tfrac{1}{Z}\pi_\psi\pi_\phi$ are not invariant policies.
\end{proof}

\section{Analysis of other choices of policy fusion}
\label{sec:other_choice}
In this section, we focus on analyzing two common fusion methods: the weighted average and product fusion methods through the lens of personalisation. The fusion strategy should comply with two constraints as defined in the main paper. Firstly, the fused policy should act on the common support of the policies being fused and secondly, satisfy the invariability constraint. 

The former method is defined as $\tfrac{1}{Z}\alpha\pi_1 + (1-\alpha)\pi_2$, where $\pi_1$ and $\pi_2$ are the policies to be fused, $Z$ is normalising factor and $\alpha$ is a weight parameter. This method satisfies the invariability property, ensuring that the fused policy remains unchanged when both input policies are similar. However, this fusion strategy acts on the union of the support of the input policies, meaning that actions are selected to maximize either one objective and not both simultaneously.

The latter method involves taking the product of the two policies: $\tfrac{1}{Z}\pi_f(a|s) = \pi_1(a|s) \times \pi_2(a|s)$. While this fusion operates on the intersection of the support of the input policies, ensuring that actions are selected to maximize both objectives simultaneously, it fails to satisfy the invariability property as shown in Lemma \ref{lm:invariability},

However, the fused policies produced by the product can be bounded as captured by the following theorem,

\begin{theorem}
    Let $Q$ and $Q'$ be Q-values corresponding to the task policy and the intent specific policy respectively. Let $\pi_\psi(a|s)$ and $\pi_\phi(a|s)$ represent the respective policies, with corresponding temperatures $T_\psi$ and $T_\phi$. Let $\norm{Q(s,a) - Q'(s,a)}_2 < \epsilon \ \forall s\in \mathcal{S} \ and \ a\in\mathcal{A}$  and $\norm{T_\psi-T_\phi}_2 < \delta$. Then, 
    \begin{center}
        $KL\left(\pi_\phi(a|s) \middle\| \tfrac{\pi_\phi(a|s)\pi_\psi(a|s)}{Z}\right) \leq\log \left( Z\right) + 
            \left(\dfrac{Q^{*}\delta + \epsilon T_\phi}{T_\phi T_\psi} \right)
            +\log\left(\zeta\right) \ \forall a\in\mathcal{A}, s\in\mathcal{S}$,
    \end{center}
    where $Z$ is its normalising factor of $\pi_f(a|s)$, $Q^{*} = max_{a \in \mathcal{A}}Q(s, a)$ and $\zeta = \dfrac{h(Q',T_\psi)}{h(Q,T_\phi)}$, here $h(Q,T) = \sum_aexp\left(\dfrac{Q}{T}\right)$.
\end{theorem}
\begin{proof}
    \begin{align*}
        &KL\left(\pi_\phi \middle\| \pi_f\right)=
        \sum _{a}\pi _{\phi }\log \left( \dfrac{\pi_\phi Z}{\pi_f}\right) \\
        &=\log \left( Z\right) + \sum _{a}\pi_\phi
        \left(\log\left(e^{\dfrac{Q}{T_\phi}-\dfrac{Q'}{T_\psi}}\right) +\log\left(\zeta\right)\right) \\
        &=\log \left( Z\right) + \sum _{a}\pi_\phi
        \left(\dfrac{Q}{T_\phi}-\dfrac{Q'}{T_\psi} +\log\left(\zeta\right)\right) \\
        &=\log \left( Z\right) + \sum _{a}\pi_\phi
        \left(\dfrac{Q}{T_\phi}-\dfrac{Q'}{T_\psi} \right)
        +\log\left(\zeta\right) \\
        &=\log \left( Z\right) + \sum _{a}\pi_\phi
        \left(\dfrac{QT_\psi-Q'T_\phi}{T_\phi T_\psi} \right)
        +\log\left(\zeta\right) \\
        &=\log \left( Z\right) + \sum _{a}\pi_\phi
        \left(\dfrac{QT_\psi-Q'T_\phi +QT_\phi - QT_\phi}{T_\phi T_\psi} \right) \\
        & \qquad +\log\left(\zeta\right) \\
        &\leq\log \left( Z\right) + \sum _{a}\pi_\phi
        \left(\dfrac{Q\delta + \epsilon T_\phi}{T_\phi T_\psi} \right) \\
        & \qquad +\log\left(\zeta\right)\\
        &\leq\log \left( Z\right) + 
        \left(\dfrac{Q^{*}\delta + \epsilon T_\phi}{T_\phi T_\psi} \right)
        +\log\left(\zeta\right)
    \end{align*}
\end{proof}

\section{MORL baseline Implementation Details}
\label{sec:morl_impl}
 To implement the MORL baseline, we construct the reward as a convex sum of the rewards of both objectives as $r = \alpha r_e + (1-\alpha)r_h$, where $r_e$ is the reward obtained from the environment and the $r_h$ is the reward obtained from the LSTM as in Equation \eqref{eq:reward_redistribution}. To match with our zero-shot approach, we change the reward of the transitions of the trajectories saved for the training of LSTM. We then populated the new transitions in the replay buffer and trained a new DQN. We choose $\alpha = 0.5$ in all our experiments to balance both objective and normalised the human reward components between $[-1,1]$. 

\begin{table*}[h]
\caption{\label{tab:2d_static_additional} Results of static fusion when $T_\psi = T_{\min}$ in 2D Navigation. $\uparrow$ indicates higher values are better and $\downarrow$ indicates lower values are better. Results averaged over 10 seeds, each with 300 episodes. Hyperparameters: $T_{\phi} = 0.4$, $T_{\min} = 1$, $T_{\max} = 10$, $\eta = 0$.}
\centering
\begin{tabular}{l l c c c}
\toprule
Mode & Method & \begin{tabular}{@{}c@{}}Desired region \\ visits $\uparrow$\end{tabular} & \begin{tabular}{@{}c@{}}Undesired region \\ visits $\downarrow$\end{tabular} & \begin{tabular}{@{}c@{}}Score \\ $\uparrow$\end{tabular} \\
\midrule
Preference & Static & $4.987 \pm 0.649$ & -- & $0.284 \pm 0.045$ \\
Mixed      & Static & $2.372 \pm 0.123$ & $0.000 \pm 0.000$ & $0.739 \pm 0.014$ \\
Avoidance  & Static & -- & $0.000 \pm 0.000$ & $0.974 \pm 0.011$ \\
\bottomrule
\end{tabular}
\end{table*}

\begin{table*}[h]
\caption{\label{tab:highway_additional} Results in the Highway environment with different values of $\alpha$. $\uparrow$ indicates higher values are better and $\downarrow$ indicates lower values are better. Hyperparameters: $T_{\phi} = 0.6$, $T_{\max} = 5$, $T_{\min} = 0.3$, $\eta = 0$.}
\centering
\begin{tabular}{l c c c c c}
\toprule
Mode & $\alpha$ & \begin{tabular}{@{}c@{}}Desired lane \\ visits $\uparrow$\end{tabular} & \begin{tabular}{@{}c@{}}Undesired lane \\ visits $\downarrow$\end{tabular} & \begin{tabular}{@{}c@{}}Hits \\ $\downarrow$\end{tabular} & \begin{tabular}{@{}c@{}}Score \\ $\uparrow$\end{tabular} \\
\midrule
\multirow{3}{*}{Avoidance} 
& 0.3 & -- & $16.25 \pm 3.45$ & $0.15 \pm 0.03$ & $37.95 \pm 0.68$ \\
& 0.5 & -- & $5.80 \pm 1.07$  & $0.14 \pm 0.23$ & $37.44 \pm 0.52$ \\
& 0.7 & -- & $2.75 \pm 0.49$  & $0.47 \pm 0.02$ & $33.90 \pm 0.91$ \\
\midrule
\multirow{3}{*}{Preference} 
& 0.3 & $9.64 \pm 0.79$  & -- & $0.15 \pm 0.02$ & $37.58 \pm 0.60$ \\
& 0.5 & $11.58 \pm 1.12$ & -- & $0.07 \pm 0.06$ & $38.14 \pm 0.30$ \\
& 0.7 & $23.74 \pm 1.58$ & -- & $0.09 \pm 0.01$ & $38.42 \pm 0.49$ \\
\midrule
\multirow{3}{*}{Mixed} 
& 0.3 & $11.90 \pm 0.93$ & $11.45 \pm 1.27$ & $0.13 \pm 0.02$ & $38.03 \pm 0.55$ \\
& 0.5 & $12.61 \pm 1.20$ & $12.36 \pm 1.58$ & $0.07 \pm 0.01$ & $36.97 \pm 0.59$ \\
& 0.7 & $23.18 \pm 2.55$ & $6.13 \pm 1.47$  & $0.11 \pm 0.03$ & $37.64 \pm 0.51$ \\
\bottomrule
\end{tabular}
\end{table*}

\section{Additional experiments}
\label{sec:additional_exp}

In this section, we report a set of supplementary experiments designed to provide further insights into the behaviour of our method and baselines.  

Firstly, we examined the effect of fixing the LSTM temperature ($T_\psi$) at its minimum value $T_{\min}$ in the 2D Navigation environment. As shown in Table~\ref{tab:2d_static_additional}, the fused policy becomes dominated by the human objective. This is expected, since a low temperature sharpens the Boltzmann distribution, amplifying the influence of the intent-specific policy relative to the task policy.  

Secondly, we analysed the sensitivity of the MORL baseline to different reward weightings ($\alpha$). Because MORL optimises multiple objectives simultaneously, the choice of $\alpha$ strongly affects the trade-off between task and human rewards. In the Highway environment, we tested $\alpha \in \{0.3, 0.5, 0.7\}$. As reported in Table~\ref{tab:highway_additional}, increasing $\alpha$ shifts the policy more towards the human objective, confirming the intuitive role of weighting in MORL.  

\noindent\textbf{Effect of $\eta$:}  
Table~\ref{tab:highway_eta} presents results for different values of $\eta$ under dynamic fusion. Theoretically, a higher $\eta$ increases the likelihood that $T_{\psi}$ remains low (see Fig~\ref{fig:sigmoid}), making the intent-specific policy sharper and more dominant during fusion. The empirical results show a trend consistent with this: in the Preference mode, higher $\eta$ is associated with more visits to the desired lane and fewer to the undesired lane. However, the changes across $\eta$ are relatively modest, and task scores remain broadly similar. In our main experiments we set $\eta=0$, though in other environments an alternative choice may be preferable.  

\noindent\textbf{Effect of $T_{\max}$:} 
We also investigated the influence of $T_{\max}$, which controls the maximum temperature in Equation~\eqref{eq:temp}. A larger $T_{\max}$ should, in principle, weaken the effect of the intent-specific policy. The results in Table~\ref{tab:t_max} are consistent with this expectation: as $T_{\max}$ increases, the agent shows a reduced tendency to follow the preferred lane and a greater tendency to deviate into undesired lanes. In other words, the level of adherence to human preferences decreases as $T_{\max}$ grows.   

Overall, these additional studies confirm the theoretical roles of $T_\psi$, $\alpha$, $\eta$, and $T_{\max}$ in shaping behaviour, while also showing that the practical impact of $\eta$ and $T_{\max}$ on performance is limited in the tested settings.

\begin{table*}[ht]
\caption{\label{tab:highway_eta} Results in the Highway environment with different values of $\eta$. $\uparrow$ indicates higher values are better and  $\downarrow$ indicates lower values are better for the quantity specified in the column. The hyperparameters are $T_{\phi} = 0.6$, $T_{\max} = 5$, $T_{\min} = 0.3$.}
\centering
\begin{tabular}{l c c c c c}
\toprule
Mode & $\eta$ & \begin{tabular}{@{}c@{}}Desired lane \\ visits $\uparrow$\end{tabular} & \begin{tabular}{@{}c@{}}Undesired lane \\ visits $\downarrow$\end{tabular} & \begin{tabular}{@{}c@{}}Hits \\ $\downarrow$\end{tabular} & \begin{tabular}{@{}c@{}}Score \\ $\uparrow$\end{tabular} \\
\midrule
\multirow{3}{*}{Avoidance} 
& 0 & -- & $0.28 \pm 0.08$ & $0.06 \pm 0.01$ & $39.83 \pm 0.59$ \\
& 1 & -- & $0.35 \pm 0.12$ & $0.08 \pm 0.02$ & $38.96 \pm 0.87$ \\
& 2 & -- & $0.19 \pm 0.11$ & $0.09 \pm 0.02$ & $38.80 \pm 0.51$ \\
\midrule
\multirow{3}{*}{Preference} 
& 0 & $25.14 \pm 1.44$ & -- & $0.05 \pm 0.02$ & $40.02 \pm 0.61$ \\
& 1 & $24.90 \pm 1.52$ & -- & $0.10 \pm 0.03$ & $38.84 \pm 0.84$ \\
& 2 & $29.91 \pm 1.38$ & -- & $0.06 \pm 0.01$ & $39.27 \pm 0.59$ \\
\midrule
\multirow{3}{*}{Mixed} 
& 0 & $21.52 \pm 2.00$ & $0.85 \pm 0.33$ & $0.12 \pm 0.01$ & $38.50 \pm 0.43$ \\
& 1 & $23.96 \pm 1.67$ & $0.56 \pm 0.13$ & $0.07 \pm 0.01$ & $39.56 \pm 0.45$ \\
& 2 & $25.21 \pm 2.55$ & $0.27 \pm 0.06$ & $0.07 \pm 0.04$ & $39.15 \pm 0.49$ \\
\bottomrule
\end{tabular}
\end{table*}

\begin{table*}[h]
\caption{\label{tab:t_max} Sensitivity analysis of $T_{\max}$ in the Highway environment (Mixed mode). The hyperparameters are $T_{\phi} = 0.6$, $T_{\min} = 0.3$, and $\eta=0$. $\uparrow$ indicates higher values are better and $\downarrow$ indicates lower values are better.}
\centering
\begin{tabular}{c c c c c}
\toprule
$T_{\max}$ & \begin{tabular}{@{}c@{}}Desired lane \\ visits $\uparrow$\end{tabular} & \begin{tabular}{@{}c@{}}Undesired lane \\ visits $\downarrow$\end{tabular} & \begin{tabular}{@{}c@{}}Hits \\ $\downarrow$\end{tabular} & \begin{tabular}{@{}c@{}}Score \\ $\uparrow$\end{tabular} \\
\midrule
10  & $19.21 \pm 1.42$ & $2.02 \pm 0.43$ & $0.06 \pm 0.02$ & $40.12 \pm 0.59$ \\
15  & $17.60 \pm 0.90$ & $3.42 \pm 0.43$ & $0.09 \pm 0.02$ & $39.64 \pm 0.45$ \\
25  & $9.53 \pm 0.82$  & $4.65 \pm 0.52$ & $0.10 \pm 0.02$ & $39.51 \pm 0.53$ \\
35  & $14.36 \pm 0.65$ & $5.32 \pm 0.58$ & $0.10 \pm 0.02$ & $40.03 \pm 0.59$ \\
\bottomrule
\end{tabular}
\end{table*}

\section{DQN Training}

DQN was utilised to obtain the task-specific policy. For the 2D navigation environment, which involves image-based observations, we adopted the Convolutional Neural Network (CNN) network as detailed in Table \ref{tab:2d_dqn}. In the case of the Pong and Highway environments, an MLP was implemented, featuring three hidden layers with sizes of 256, 512, and 512. We use the ReLU activation function after each layer. We set the minimum epsilon at 0.10, with a decay rate of 0.995. 

\begin{table*}[h]
\caption{\label{tab:2d_dqn} CNN architecture of the DQN agent.}
\centering
\begin{tabular}{l l}
\toprule
\textbf{Particulars} & \textbf{Value} \\
\midrule
Channels        & 32, 64, 64 \\
Kernel sizes    & 11, 5, 3 \\
Strides         & 2, 2, 1 \\
Paddings        & 0, 0, 0 \\
FC hidden layer & 512 \\
\bottomrule
\end{tabular}
\end{table*}



\end{appendices}
\clearpage

\end{document}